%% file: main.tex
\documentclass{article}

\usepackage{microtype}
\usepackage{graphicx}
\usepackage{caption}
\usepackage{subcaption}
\usepackage{booktabs} 
\usepackage{pifont}
\usepackage{hyperref}

\usepackage[accepted]{icml2024}

\usepackage{amsmath}
\usepackage{amssymb}
\usepackage{mathtools}
\usepackage{amsthm}
\usepackage{thmtools}
\usepackage{thm-restate}

\usepackage{multirow}
\usepackage{adjustbox}
\usepackage{wrapfig}
\newcommand*{\Resize}[2]{\resizebox{0.91\linewidth}{!}{$#1$}}
\makeatletter
\newcommand\footnoteref[1]{\protected@xdef\@thefnmark{\ref{#1}}\@footnotemark}
\makeatother

\usepackage[capitalize]{cleveref}
 \AtBeginDocument{%
    \crefname{figure}{Figure}{Figures}%
}
\newcommand{\eqlabelleft}{(}
\newcommand{\eqlabelright}{)}

\creflabelformat{equation}{\eqlabelleft#2#1#3\eqlabelright}

\renewcommand{\algorithmicrequire}{\textbf{Input:}}
\renewcommand{\algorithmicensure}{\textbf{Output:}}

\theoremstyle{plain}

\theoremstyle{definition}

\theoremstyle{remark}

\usepackage[textsize=tiny]{todonotes}

\input{math_commands.tex}

\icmltitlerunning{Diffusion Rejection Sampling}

\begin{document}

\twocolumn[
\icmltitle{Diffusion Rejection Sampling}



\icmlsetsymbol{equal}{*}

\begin{icmlauthorlist}
\icmlauthor{Byeonghu Na}{kaist}
\icmlauthor{Yeongmin Kim}{kaist}
\icmlauthor{Minsang Park}{kaist}
\icmlauthor{Donghyeok Shin}{kaist}
\icmlauthor{Wanmo Kang}{kaist}
\icmlauthor{Il-Chul Moon}{kaist,summary}
\end{icmlauthorlist}

\icmlaffiliation{kaist}{Department of Industrial \& Systems Engineering, KAIST, Daejeon, Republic of Korea}
\icmlaffiliation{summary}{summary.ai, Daejeon, Republic of Korea}

\icmlcorrespondingauthor{Il-Chul Moon}{icmoon@kaist.ac.kr}
\icmlcorrespondingauthor{Byeonghu Na}{byeonghu.na@kaist.ac.kr}

\icmlkeywords{Machine Learning, ICML}

\vskip 0.3in
]



\printAffiliationsAndNotice{}  

\begin{abstract}
Recent advances in powerful pre-trained diffusion models encourage the development of methods to improve the sampling performance under well-trained diffusion models. This paper introduces Diffusion Rejection Sampling (DiffRS), which uses a rejection sampling scheme that aligns the sampling transition kernels with the true ones at each timestep. The proposed method can be viewed as a mechanism that evaluates the quality of samples at each intermediate timestep and refines them with varying effort depending on the sample.
Theoretical analysis shows that DiffRS can achieve a tighter bound on sampling error compared to pre-trained models. Empirical results demonstrate the state-of-the-art performance of DiffRS on the benchmark datasets and the effectiveness of DiffRS for fast diffusion samplers and large-scale text-to-image diffusion models. Our code is available at \url{https://github.com/aailabkaist/DiffRS}.
\end{abstract}

\section{Introduction}
\label{sec:intro}


Diffusion models have attracted considerable interest in various domains, such as image~\cite{dhariwal2021diffusion,rombach2022high} and video generation~\cite{ho2022video,voleti2022mcvd}, due to their remarkable ability to generate high-quality samples. The powerful generative capabilities of diffusion models have spurred extensive efforts to further improve the sampling quality.
A common strategy is to reduce the sampling interval, thereby increasing the iterative sampling count~\cite{karras2022elucidating}. However, this comes at the cost of a higher number of network evaluations, which slows down the sampling speed.
An alternative approach is to improve the training of the reverse diffusion process to accurately model the reverse transition~\cite{kim2022soft,rombach2022high,lai2023fp,zheng2023truncated}. Nonetheless, these methods require time-consuming training of the diffusion model.

In contrast to these approaches, recent advances in powerful pre-trained models~\cite{rombach2022high,karras2022elucidating} have led to a growing body of research focused on leveraging them~\cite{kim2023refining,xu2023restart,ning2024elucidating}. In line with these efforts, our goal is to effectively and efficiently leverage a well-trained diffusion model to improve the sampling quality. We introduce a mechanism that assesses the quality of a sample at each intermediate timestep, allowing us to keep good samples as well as to refine poor samples by injecting appropriate noise and by going back to prior timesteps.

Specifically, we propose Diffusion Rejection Sampling (DiffRS), which is based on the ratio of the true transition kernel to the transition kernel of the pre-trained model for each timestep, see \cref{fig:overview}. The ratio can be estimated by a time-dependent discriminator that distinguishes between data and generated samples at each timestep. In cases where samples are rejected, we adjust the noise intensity depending on the rejected samples. We theoretically prove that discriminator training leads to a tighter upper bound on the sampling error of DiffRS compared to a pre-trained diffusion model.
In the experiments, DiffRS achieves new state-of-the-art (SOTA) performance on CIFAR-10, and near-SOTA performance on ImageNet 64$\times$64 with fewer NFEs. Moreover, we demonstrate the effective application of DiffRS to the fast diffusion samplers, such as DPM-Solver++~\cite{lu2022dpm++} and Consistency Model~\cite{song2023consistency}, and large-scale text-to-image generation models, including Stable Diffusion~\cite{rombach2022high}.


\begin{figure*}[t]
    \centering
    \includegraphics[width=0.895\linewidth]{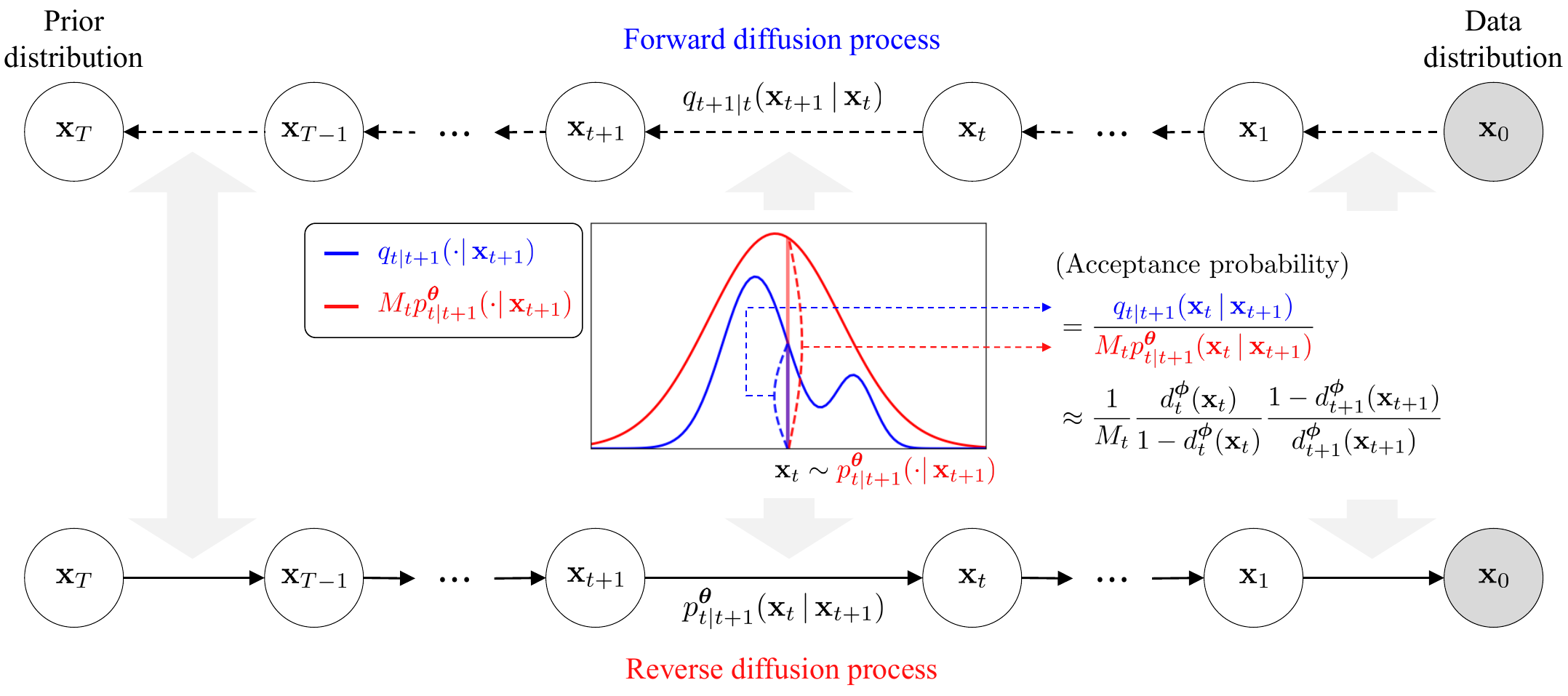}
    \caption{Overview of DiffRS. We sequentially apply the rejection sampling on the pre-trained transition kernel $p^{\boldsymbol{\theta}}_{t|t+1}(\rvx_t |\rvx_{t+1})$ (red) to align the true transition kernel $q_{t|t+1}(\rvx_t | \rvx_{t+1})$ (blue). The acceptance probability is estimated by the time-dependent discriminator $d^{\boldsymbol{\phi}}_t$.}
    \label{fig:overview}
\end{figure*}

\section{Preliminary}
\label{sec:prelim}

\textbf{Diffusion Model}~
Diffusion-based generative models~\cite{ho2020denoising,song2021scorebased,dhariwal2021diffusion} are one of the most prominent deep generative models that aim to approximate the data distribution to the model distribution. This model includes a forward diffusion process that iteratively perturbs the data instances toward the prior distribution, and a corresponding reverse process that inverts the forward process to sample from the modeled distribution.

The forward process is formulated by a fixed Markov chain that constructs a set of latent variables $\rvx_{1:T}$ by adding Gaussian noises from data distribution $q_0(\rvx_0)$~\cite{ho2020denoising}:
\begin{align}
\label{eq:forward_ddpm}
    q(\rvx_{1:T}|\rvx_0) := \textstyle\prod_{t=1}^{T} q_{t|t-1}(\rvx_t | \rvx_{t-1}),
\end{align}
where $q_{t|t-1}(\rvx_t|\rvx_{t-1}) := \mathcal{N}(\rvx_t; \sqrt{1-\beta_{t}} \rvx_{t-1}, \beta_t \mathbf{I})$ and $\beta_t$ is a variance schedule parameter at time $t$.
Most diffusion models define the reverse process by a Markov chain with a Gaussian transition kernel $p_{t|t+1}(\rvx_{t}|\rvx_{t+1})$:
\begin{align}
\label{eq:backward_ddpm}
    p(\rvx_{0:T}) := p_T(\rvx_{T}) \textstyle\prod_{t=0}^{T-1} p_{t|t+1}(\rvx_{t} | \rvx_{t+1}),
\end{align}
where $p_T(\rvx_T)$ is the prior distribution.
Then, the goal of the diffusion model is to approximate the transition kernel $p_{t|t+1}(\rvx_{t}|\rvx_{t+1})$ by a Gaussian with parameterized mean $\boldsymbol{\mu}^{\boldsymbol{\theta}}$ and time-dependent variance $\sigma_{t+1}^2$,
\begin{align}
\label{eq:kernel}
    p_{t|t+1}^{\boldsymbol{\theta}}(\rvx_{t}|\rvx_{t+1}) := \mathcal{N}(\rvx_{t}; \boldsymbol{\mu}^{\boldsymbol{\theta}}(\rvx_{t+1}, t+1), \sigma_{t+1}^2 \mathbf{I}),
\end{align}
using the objective of variational bound on the log likelihood.
When the parameterized transition kernel $p_{t|t+1}^{\boldsymbol{\theta}}$ is obtained, we proceed with iterative sampling from $T$ to $0$ using \cref{eq:backward_ddpm}, replacing a transition kernel with $p_{t|t+1}^{\boldsymbol{\theta}}$:
\begin{align}
\label{eq:sampling}
    \Resize{\rvx_{t} = \boldsymbol{\mu}^{\boldsymbol{\theta}}(\rvx_{t+1}, t+1) + \sigma_{t+1}^2 \rvz \text{ where } \rvz \sim \mathcal{N}(\rvz;\mathbf{0},\mathbf{I}).}.
\end{align}



\textbf{Refining Sampling Process from Pre-trained Models}
~While most previous methods require training of diffusion models to reduce the sampling error, some recent work has explored refining the sampling process from pre-trained diffusion models.
DG~\cite{kim2023refining} corrects the transition kernel by adding an auxiliary term from the discriminator $d_t^{\boldsymbol{\phi}}$ that distinguishes between real and generated samples:
\begin{align}
    \boldsymbol{\mu}^{\boldsymbol{\theta},\boldsymbol{\phi}}(\rvx_{t}, t) := \boldsymbol{\mu}^{\boldsymbol{\theta}}(\rvx_{t}, t) + \alpha_t \nabla_{\rvx_{t}} \log \tfrac{d^{\boldsymbol{\phi}}_t (\rvx_t)}{1-d^{\boldsymbol{\phi}}_t (\rvx_t)},
\end{align}
where $\alpha_t$ is a time-dependent constant. After that, sampling proceeds with the adjusted transition kernel $\boldsymbol{\mu}^{\boldsymbol{\theta},\boldsymbol{\phi}}$ to reduce the network estimation error.
We also use a fixed pre-trained diffusion and utilize the discriminator, but our distinctive method is the application of a rejection sampling scheme.

In addition, Restart~\cite{xu2023restart} introduces a strategy of repeating the backward and forward steps at fixed time interval $[t_{\text{min}},t_{\text{max}}]$.
Specifically, Restart iteratively samples with a deterministic sampler, such as an ODE sampler, from $T$ to $t_{\text{min}}$. Then, it imposes stochasticity by adding large noise and simulates a reverse process from $t_{\text{max}}$ to $t_{\text{min}}$:
\begin{align}
    (\text{Restart forward}) & ~ \rvx_{t_{\text{max}}}^{i+1} = \rvx^i_{t_{\text{min}}} + \epsilon_{t_{\text{min}} \rightarrow t_{\text{max}}}, \\
    (\text{Restart reverse}) ~& ~ \rvx_{t_{\text{min}}}^{i+1} = \text{ODE}_{\boldsymbol{\theta}}(\rvx^{i+1}_{t_{\text{max}}}, t_{\text{max}} \rightarrow t_{\text{min}}),
\end{align}
where $\epsilon_{t_{\text{min}} \rightarrow t_{\text{max}}}$ denotes the injected noise of the forward process from $t_{\text{min}}$ to $t_{\text{max}}$ and $\text{ODE}_{\boldsymbol{\theta}}$ represents the reverse process using a deterministic sampler from  $t_{\text{max}}$ to $t_{\text{min}}$.
These processes are repeated, demonstrating an increased contraction effect on accumulated errors.
Our rejection sampling differs in that the timesteps for applying the forward process are determined probabilistically for each sample.

\textbf{Rejection Sampling}~
Rejection sampling is a numerical sampling method to be used when a target distribution $q(\rvx)$ can be evaluated whereas its direct sampling is difficult~\cite{ripley2009stochastic}. For this, we need a proposal distribution $p(\rvx)$ that can be evaluated and from which we can draw samples. We also need to find a constant $M$ satisfying $q(\rvx) \leq Mp(\rvx)$ for all $\rvx$. Then, we accept a sample $\rvx$ drawn from $p(\rvx)$ with probability of ${q(\rvx)}/{Mp(\rvx)}$, and otherwise reject it.

Some work on generative models takes advantage of this rejection sampling scheme. \citet{grover2018variational} use it to improve samples drawn from the variational posterior of the variational autoencoder. \citet{azadi2018discriminator,turner2019metropolis} generate data instances from the generative adversarial network by evaluating the acceptance probability using the discriminator. Compared to previous studies, the sampling of diffusion models is iterative, which requires a sequential rejection sampling method over diffusion timesteps. 

\begin{algorithm}[tb]
    \caption{$\mathtt{OneStepDiffRS}$ ($t, \rvx_{t+1}, L_{t+1}$)}
    \label{alg:onestep}
    \begin{algorithmic}[1]
        \REQUIRE $p_{t|t+1}^{\boldsymbol{\theta}}$, $q_t/p_t^{\boldsymbol{\theta}}$ (or ${d^{\boldsymbol{\phi}}_t }/{[1-d^{\boldsymbol{\phi}}_t ]}$), $M_t$
        \ENSURE $\rvx_{t}$, $L_t$

        \STATE $\rvx_{t} \leftarrow$ None
        \WHILE{$\rvx_{t}$ is None}
            \STATE Sample $\tilde{\rvx}_t$ from the transition kernel $p_{t|t+1}^{\boldsymbol{\theta}} (\cdot | \rvx_{t+1})$
            \STATE Compute $L_t \leftarrow \frac{q_t (\tilde{\rvx}_t)}{p_t^{\boldsymbol{\theta}} (\tilde{\rvx}_t)}$ and $A_t \leftarrow \frac{L_t}{M_t L_{t+1}}$
            \STATE Sample $u \sim \text{Uniform}(0,1)$
            \IF{$u < A_t$}
                \STATE $\rvx_t \leftarrow \tilde{\rvx}_t$
            \ELSE
                \STATE $\rvx_{t+1}, L_{t+1} \leftarrow$ $\mathtt{Re}$-$\mathtt{initialization}$($t+1, \tilde{\rvx}_t$)
            \ENDIF
        \ENDWHILE
    \end{algorithmic}
\end{algorithm}

\section{Methods}
\label{sec:method}

\subsection{Diffusion Rejection Sampling (DiffRS)}
\label{subsec:diffrs}
We assume the existence of a pre-trained diffusion model that allows the generation of samples using the transition kernel $p_{t|t+1}^{\boldsymbol{\theta}}(\rvx_{t}|\rvx_{t+1})$. The distribution of a sample $\rvx_0$ obtained through a sequence of transition samples $\rvx_t|\rvx_{t+1}$, denoted $p_0^{\boldsymbol{\theta}}(\rvx_0)$, may deviate from the true data distribution $q_0(\rvx_0)$ if the pre-trained transition kernel $p_{t|t+1}^{\boldsymbol{\theta}}$ differs from the true transition kernel $q_{t|t+1}$. Consequently, we apply a rejection sampling scheme for each timestep in the transition kernel to mitigate this discrepancy, as described in \cref{fig:overview}.

Conceptually, DiffRS performs the rejection sampling of the transition probability in reverse diffusion, $p_{t|t+1}^{\boldsymbol{\theta}}$.\footnote{It should be noted that the rejection sampling is imposed on the transition probability, $p_{t|t+1}^{\boldsymbol{\theta}}$; not its marginal probability, $p_t^{\boldsymbol{\theta}}$.} During the generation procedure, the sampling means selecting an instance from $p_{t|t+1}^{\boldsymbol{\theta}}$, which follows a Gaussian of \cref{eq:kernel}, so it can perform as a proposal distribution of the rejection sampling. Meanwhile, the ordinary forward diffusion, $q_{t+1|t}$, follows a Gaussian distribution; but its reverse-time version, $q_{t|t+1}$, does not follow a Gaussian distribution, which becomes the target distribution of the rejection sampling.

To formulate DiffRS, let $q_t(\rvx_t)$ and $p_t^{\boldsymbol{\theta}}(\rvx_t)$ represent the marginal distributions of the forward diffusion process starting from $q_0(\rvx_0)$ and $p_0^{\boldsymbol{\theta}}(\rvx_0)$, respectively. We introduce a one-step DiffRS procedure from $t+1$ to $t$ to obtain a sample $\rvx_{t}$ from $q_t$, given a sample $\rvx_{t+1}$ from $q_{t+1}$. This procedure can be applied sequentially from $T-1$ to $0$, yielding a sample from the data distribution $q_0$.

\textbf{Proposal Distribution}~
At time $t+1$, we assume that we have a sample $\rvx_{t+1}$ drawn from the perturbed data distribution $q_{t+1}(\rvx_{t+1})$ through the sampling iterations from $T$ to $t+1$. Then, a sample $\rvx_t$ at time $t$ can be drawn using the pre-trained transition kernel $p_{t|t+1}^{\boldsymbol{\theta}}$ following the generative reverse process by \cref{eq:sampling}. Our goal is to ensure that the sampling closely follows the true transition kernel $q_{t|t+1}$ (blue in \cref{fig:overview}). This is achieved by applying the rejection sampling, where the proposal distribution is set by the pre-trained transition kernel $p_{t|t+1}^{\boldsymbol{\theta}}$ (red in \cref{fig:overview}).

\textbf{Acceptance Probability}~
To implement the rejection sampling scheme on the transition kernel, we need to compute the acceptance probability $A_{t}(\rvx_t,\rvx_{t+1})$, which is expressed as the ratio of the true and pre-trained transition kernel:
\begin{align}
\label{eq:accept1}
    A_{t}(\rvx_t,\rvx_{t+1}) := \frac{1}{M_t} \frac{q_{t|t+1}(\rvx_t|\rvx_{t+1})}{ p_{t|t+1}^{\boldsymbol{\theta}}(\rvx_t|\rvx_{t+1})},
\end{align}
where $M_t$ is a constant that satisfies $q_{t|t+1}(\rvx_t|\rvx_{t+1}) \leq M_t p_{t|t+1}^{\boldsymbol{\theta}}(\rvx_t|\rvx_{t+1})$ for all $\rvx_t$ and $\rvx_{t+1}$. The density ratio can be further derived as follows:
\begin{small}
\begin{align}
\label{eq:accept2}
    \frac{q_{t|t+1}(\rvx_t|\rvx_{t+1})}{ p_{t|t+1}^{\boldsymbol{\theta}}(\rvx_t|\rvx_{t+1})} & = \frac{q_{t+1|t}(\rvx_{t+1}|\rvx_{t})}{p_{t+1|t}(\rvx_{t+1}|\rvx_{t})} \frac{q_t(\rvx_t)}{p_{t}^{\boldsymbol{\theta}} (\rvx_t)} \frac{ p_{t+1}^{\boldsymbol{\theta}}(\rvx_{t+1})}{q_{t+1} (\rvx_{t+1})} \nonumber \\
    & = \frac{q_t(\rvx_t)}{p_{t}^{\boldsymbol{\theta}} (\rvx_t)} \frac{ p_{t+1}^{\boldsymbol{\theta}}(\rvx_{t+1})}{q_{t+1} (\rvx_{t+1})} = \frac{L_t(\rvx_t)}{L_{t+1}(\rvx_{t+1})},
\end{align}
\end{small}

where $L_t(\rvx_t) := \frac{q_t(\rvx_t)}{p_{t}^{\boldsymbol{\theta}} (\rvx_t)}$. The first equality holds by Bayes' rule, and we use the fact that the perturbed kernels, $q_{t+1|t}$ and $p_{t+1|t}$, are the same for the second equality. Therefore, the acceptance probability $A_t(\rvx_t, \rvx_{t+1})$ of the one-step DiffRS at time $t$ can be expressed as follows:
\begin{align}
\label{eq:accept3}
    A_t(\rvx_t, \rvx_{t+1}) = \frac{L_t(\rvx_t)}{M_t L_{t+1}(\rvx_{t+1})}.
\end{align}
$L_t(\rvx_t)$ is estimated by the density ratio estimation via a discriminator $d^{\boldsymbol{\phi}}_t$, which will be discussed in \cref{subsec:accept}.

\textbf{Algorithm of One-step DiffRS}~
We formulate a one-step DiffRS procedure in \cref{alg:onestep}. Note that Re-initialization (line 9 in \cref{alg:onestep}) refers to the process of drawing a new sample at timestep $t+1$ after a rejection, which we will explain further in \cref{subsec:reinit}.

\subsection{Estimation of the Acceptance Probability}
\label{subsec:accept}

As indicated in \cref{eq:accept3}, the acceptance probability is expressed as the ratio of the likelihood ratios at time $t$ and $t+1$. Therefore, if we can estimate the likelihood ratio $L_t(\rvx_t)$ at each timestep, we can compute the acceptance probability. Following the approach of DG~\cite{kim2023refining}, we estimate this ratio using a time-dependent discriminator, denoted by $d^{\boldsymbol{\phi}}_t$. This discriminator is designed to distinguish between samples of $q_t$ and $p_t^{\boldsymbol{\theta}}$ at all timesteps.

To train the discriminator, we generate the samples of $p_0^{\boldsymbol{\theta}}$ using \cref{eq:sampling} with the pre-trained diffusion model. The training objective is the time-weighted binary cross-entropy loss using the real and generated samples:

\begin{small}
    \vspace{-1em}
    \begin{align}
    \label{eq:bce}
        \mathcal{L}_{\text{BCE}}({\boldsymbol{\phi}}) := & \E_t \Big [ \lambda(t) \E_{q_0(\rvx_0)q_{t|0}(\rvx_t|\rvx_0)} \big [-\log d^{\boldsymbol{\phi}}_t(\rvx_t) \big ] \nonumber \\
        & + \E_{p_0^{\boldsymbol{\theta}}(\rvx_0)q_{t|0}(\rvx_t|\rvx_0)}\big [-\log (1-d^{\boldsymbol{\phi}}_t(\rvx_t) ) \big]\Big ],
    \end{align}
\end{small}

where $\lambda(t)$ is the temporal weighting function.
Then, the optimal discriminator $d^{\boldsymbol{\phi}^*}_t$ satisfies the following equations:
\begin{small}
\begin{align}
\label{eq:opt_disc}
    d^{\boldsymbol{\phi}^*}_t (\rvx_t) = \frac{q_t(\rvx_t)}{q_t(\rvx_t) + p_t^{\boldsymbol{\theta}}(\rvx_t)}; L_t(\rvx_t) = \frac{q_t(\rvx_t)}{p_t^{\boldsymbol{\theta}}(\rvx_t)} = 
    \frac{d^{\boldsymbol{\phi}^*}_t (\rvx_t)}{1-d^{\boldsymbol{\phi}^*}_t (\rvx_t)}.
\end{align}
\end{small}

Therefore, using the time-dependent discriminator $d^{\boldsymbol{\phi}}_t$, we derive the estimators $\hat{L}_t^{\boldsymbol{\phi}}$ and $\hat{A}_t^{\boldsymbol{\phi}}$ for the ratio $L_t$ and the acceptance probability $A_t$, respectively:
\begin{align}
\label{eq:est_disc}
    {L}_t(\rvx_t) & \approx \hat{L}_t^{\boldsymbol{\phi}}(\rvx_t)  := 
    \frac{d^{\boldsymbol{\phi}}_t (\rvx_t)}{1-d^{\boldsymbol{\phi}}_t (\rvx_t)}, \\
    A_t(\rvx_t,\rvx_{t+1}) & \approx \hat{A}_t^{\boldsymbol{\phi}}(\rvx_t, \rvx_{t+1}) := \frac{1}{M_t} \frac{ \hat{L}_t^{\boldsymbol{\phi}}(\rvx_t)}{ \hat{L}_{t+1}^{\boldsymbol{\phi}}(\rvx_{t+1})}.
\end{align}

\begin{figure}[t]
    \centering
    \includegraphics[width=0.95\linewidth]{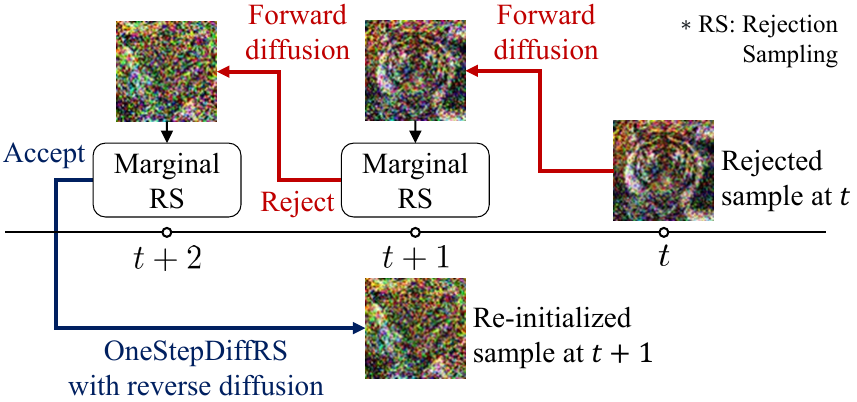}
    \caption{Overview of the proposed re-initialization.}
    \label{fig:reinit}
\end{figure}

\begin{algorithm}[tb]
   \caption{$\mathtt{Re}$-$\mathtt{initialization}$($t+1, \tilde{\rvx}_t$)}
   \label{alg:reinit}
\begin{algorithmic}[1]
    \REQUIRE $q_{t+1|t}$, $q_{t+1}/p_{t+1}^{\boldsymbol{\theta}}$ (or ${d^{\boldsymbol{\phi}}_{t+1} }/{[1-d^{\boldsymbol{\phi}}_{t+1} ]}$), $\tilde{M}_{t+1}$
    \ENSURE $\rvx_{t+1}$, $L_{t+1}$

    \STATE Sample $\tilde{\rvx}_{t+1}$ from the forward process $q_{t+1|t} (\cdot | \rvx_{t})$
    \STATE Compute ${L}_{t+1} \leftarrow \frac{q_{t+1} (\tilde{\rvx}_{t+1})}{p_{t+1}^{\boldsymbol{\theta}} (\tilde{\rvx}_{t+1})}$ and $\tilde{A}_{t+1} \leftarrow \frac{{L}_{t+1}}{\tilde{M}_{t+1}}$
    \STATE Sample $u \sim \text{Uniform}(0,1)$
    \IF{($u < \tilde{A}_{t+1}$) or ($t+1 == T$)}
        \STATE $\rvx_{t+1} \leftarrow \tilde{\rvx}_{t+1}$
    \ELSE
        \STATE $\rvx_{t+2}, L_{t+2} \leftarrow$ $\mathtt{Re}$-$\mathtt{initialization}$($t+2, \tilde{\rvx}_{t+1}$)
        \STATE $\rvx_{t+1}, L_{t+1} \leftarrow$ $\mathtt{OneStepDiffRS}$($t+1, \rvx_{t+2}, L_{t+2}$)
    \ENDIF
\end{algorithmic}
\end{algorithm}
\subsection{Re-initialization}
\label{subsec:reinit}

The primary challenge associated with the rejection sampling is the increased number of sampling iterations caused by rejections. This problem is particularly exacerbated in diffusion models that use iterative sampling, since rejections require resampling starting from the timestep $T$. To mitigate this challenge, we introduce a re-initialization method tailored for diffusion models, utilizing rejected samples.

Motivated by the observation from Restart~\cite{xu2023restart} that incorporating the forward process into the sampling process reduces the accumulated error, we add noise to the rejected samples $\rvx_t$. Unlike Restart, we inject different amounts of noise for each sample based on the likelihood ratio information we already have, as illustrated in \cref{fig:reinit}.

Specifically, we first apply a one-step forward transition $q_{t+1|t}$ to the rejected sample $\rvx_t$ at time $t$ to obtain the candidate sample $\rvx_{t+1}$ at time $t+1$. Then, we apply an additional rejection sampling procedure to the candidate sample $\rvx_{t+1}$ based on the marginal distributions $q_{t+1}$ and $p_{t+1}^{\boldsymbol{\theta}}$. If the sample is rejected again, we iterate through the one-step forward transition and the marginal rejection sampling. Consequently, the intensity of the noise is adjusted based on the probability that a rejected sample is drawn from the true distribution. We present this re-initialization procedure in \cref{alg:reinit}. Empirically, we find that this re-initialization procedure leads to effective and efficient sample generation.

\begin{algorithm}[tb]
    \caption{Diffusion Rejection Sampling (DiffRS)}
    \label{alg:diffrs}
    \begin{algorithmic}[1]

    \STATE $\rvx_T \leftarrow$ None
    \WHILE{$\rvx_{T}$ is None}
        \STATE Sample $\tilde{\rvx}_T$ from the prior distribution $p_{T}(\rvx_T)$
        \STATE Compute ${L}_{T} \leftarrow \frac{q_{T} (\rvx_{T})}{p_{T} (\rvx_{T})}$ and $\tilde{A}_{T} \leftarrow \frac{{L}_{T}}{\tilde{M}_{T}}$
        \STATE Sample $u \sim \text{Uniform}(0,1)$
        \IF{$u < \tilde{A}_{T}$}
            \STATE $\rvx_{T} \leftarrow \tilde{\rvx}_{T}$
        \ENDIF
    \ENDWHILE
    \FOR{$t=T-1$ {\bfseries to} $0$}
        \STATE $\rvx_t, L_t \leftarrow$ $\mathtt{OneStepDiffRS}$($t, \rvx_{t+1}, L_{t+1}$)
    \ENDFOR
\end{algorithmic}
\end{algorithm}

\begin{figure*}[t]
    \centering
    \includegraphics[width=0.9\linewidth]{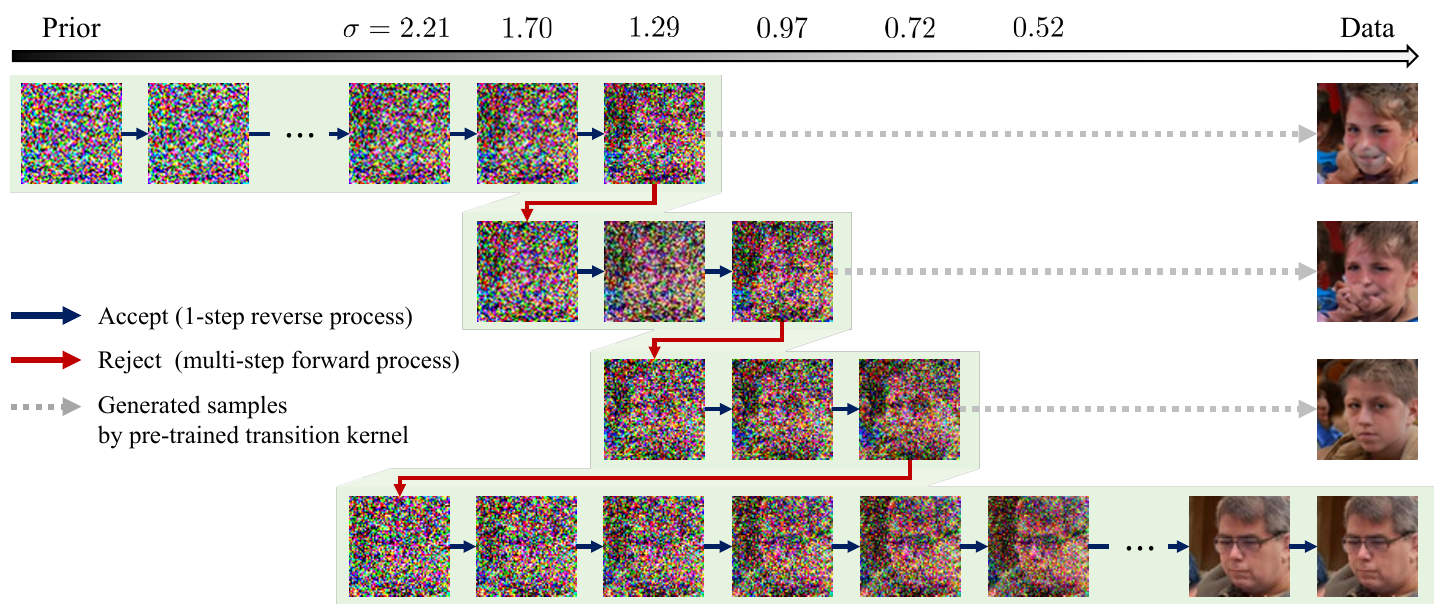}
    \caption{Illustration of the sampling process for DiffRS. The path with the green background represents the DiffRS sampling process, and the rightmost images are generated from the intermediate images using a base sampler without rejection. Timesteps are expressed as the noise level $\sigma$ from the EDM scheme~\cite{karras2022elucidating}.}
    \label{fig:process}
\end{figure*}

\subsection{Overall Algorithm}
\label{subsec:overall}

\cref{alg:diffrs} presents the overall algorithm of DiffRS. First, we sample $\rvx_T$ from the prior distribution $p_T$ and then perform the marginal rejection sampling with the acceptance probability $A_T(\rvx_T) = \frac{q_T(\rvx_T)}{M_T p_T(\rvx_T)}$ (lines 1-9). This process aims to bring the prior distribution closer to $q_T$, thereby reducing the prior mismatch error. Subsequently, we iteratively apply the one-step DiffRS from $T-1$ to $0$ (lines 10-12), ultimately obtaining a sample $\rvx_0$ on the data space.

\cref{fig:process} visually illustrates the DiffRS process, highlighted with a green background. The rightmost images show the generated samples when continuing to sample from the intermediate images without rejection. The sample is refined by finding new sampling paths through rejection.

It is important to note that DiffRS can enhance sample quality for most diffusion samplers. A necessary condition is that the sampler aims to sample from the true perturbed data distribution $q_t(\rvx_t)$ at time $t$. This condition holds true for most samplers, including diffusion distillation methods.

\textbf{Practical Consideration}~
The implementation of DiffRS requires the determination of the rejection constant $M_t$. It should be noted that $M_t$ exists for all $t$, since the diffusion process is based on a Gaussian distribution, so the support of the transition kernels becomes the entire space. However, finding an exact value for $M_t$ is nearly impossible, and even if it were possible, it would be computationally intractable in practice. In accordance with previous research~\cite{azadi2018discriminator}, we determine $M_t$ as follows: we store the ratio $\frac{\hat{L}_t^{\boldsymbol{\phi}}(\rvx_t)}{\hat{L}_{t+1}^{\boldsymbol{\phi}}(\rvx_{t+1})}$ of samples from the base sampler and select the $\gamma$\textsuperscript{th} percentile of these stored values as $M_t$. We apply this method similarly to the marginal rejection sampling.


\subsection{Theoretical Analysis}
\label{subsec:theory}

We provide a theoretical analysis of the DiffRS algorithm based on distribution divergence. \citet{ho2020denoising} derived the upper bound of the Kullback-Leibler (KL) divergence between the data distribution $q_0(\rvx_0)$ and the pre-trained distribution $p_0^{\boldsymbol{\theta}}(\rvx_0)$ in diffusion models:
\begin{align}
\label{eq:elbo}
    D_{\text{KL}} (q_0|| p_0^{\boldsymbol{\theta}}) & \leq D_{\text{KL}}(q_T|| p_T^{\boldsymbol{\theta}})  \\
    & + \sum_{t=0}^{T-1} \mathbb{E}_{q_{t+1}} \Big [ D_{\text{KL}}(q_{t|t+1} || p_{t|t+1}^{\boldsymbol{\theta}}) \Big ] =: J(\boldsymbol{\theta}). \nonumber
\end{align}

Therefore, to minimize the KL divergence on the data space, we need to match prior distributions, $q_T$ and $p^{\boldsymbol{\theta}}_T$; and transition kernels, $q_{t|t+1}$ and $p^{\boldsymbol{\theta}}_{t|t+1}$; which is the purpose of DiffRS.

For further theoretical analysis, let $p_{*}^{\boldsymbol{\theta},\boldsymbol{\phi}}$ be the distribution refined by DiffRS. We also define the unnormalized acceptance probability $\bar{A}^{\boldsymbol{\phi}}_t := M_t \hat{A}_t^{\boldsymbol{\phi}} =\frac{ \hat{L}_t^{\boldsymbol{\phi}}(\rvx_t)}{ \hat{L}_{t+1}^{\boldsymbol{\phi}}(\rvx_{t+1})}$. Then, the refined prior distribution and the refined transition kernels of DiffRS are expressed by the pre-trained distribution and the acceptance probability:
\begin{align}
    p_{T}^{\boldsymbol{\theta},\boldsymbol{\phi}}(\rvx_T) & = p_{T}^{\boldsymbol{\theta}}(\rvx_T) \bar{A}_T^{\boldsymbol{\phi}}(\rvx_T),  \label{eq:refine1} \\
    p_{t|t+1}^{\boldsymbol{\theta},\boldsymbol{\phi}}(\rvx_t|\rvx_{t+1}) & = p_{t|t+1}^{\boldsymbol{\theta}}(\rvx_t|\rvx_{t+1}) \bar{A}_t^{\boldsymbol{\phi}}(\rvx_t,\rvx_{t+1}). \label{eq:refine2}
\end{align}

\cref{thm:main} formulates the upper bound of the KL divergence between the data and refined distribution.

\begin{restatable}{theorem}{thma}
\label{thm:main}
    The KL divergence between data distribution $q_0$ and refined distribution $p_{0}^{\boldsymbol{\theta},\boldsymbol{\phi}}$ is bounded by:
    \begin{align}
    \label{eq:refine_elbo}
        D_{\text{KL}} (q_0|| p_0^{\boldsymbol{\theta},\boldsymbol{\phi}})  \leq &~ J(\boldsymbol{\theta}) + R(\boldsymbol{\phi}) =: J(\boldsymbol{\theta},\boldsymbol{\phi}),
    \end{align}
    where $R(\boldsymbol{\phi}) := \mathbb{E}_{q_{T}} [ - \log \bar{A}_T^{\boldsymbol{\phi}} ] + \sum_{t=0}^{T-1} \mathbb{E}_{q_{t,t+1}} [ - \log \bar{A}_{t}^{\boldsymbol{\phi}} ]$.
    Moreover, this bound attains equality for the optimal $\boldsymbol{\phi}^*$, and in such cases the value becomes $0$.
\end{restatable}
The proof is in \cref{sec:app_proof}. If the discriminator is completely indistinguishable, i.e., $d_t^{\boldsymbol{\phi}} \equiv 0.5$ for all $t$, then  $R(\boldsymbol{\phi})=0$ because $\bar{A}_t^{\boldsymbol{\phi}} \equiv 1$ for all $t$, indicating that all instances are accepted in the rejection sampling process. Therefore, the refined distribution from DiffRS is same as the distribution from the pre-trained diffusion model. As the discriminator is trained, $R(\boldsymbol{\phi})$ converges to $-J(\boldsymbol{\theta}) (\leq0)$ according to \cref{thm:main}, making the upper bound for DiffRS tighter than that for the pre-trained diffusion model.

\section{Experiments}
\label{sec:exp}

In this section, we empirically validate the proposed method, DiffRS. First, we conduct experiments on standard benchmark datasets for image generation tasks, such as CIFAR-10~\cite{krizhevsky2009learning}, and ImageNet 64$\times$64 and 256$\times$256~\cite{deng2009imagenet}. Next, we present the analysis of DiffRS and its applicability to fast diffusion samplers. Finally, we perform experiments on large-scale text-conditional image generation using Stable Diffusion~\cite{rombach2022high} with a resolution of 512$\times$512.

\begin{table}[t]
    \centering
    \caption{Performance comparison on CIFAR-10. The values in the first block are taken from the original paper.}
    \vskip 0.1in
    \adjustbox{max width=\linewidth}{%
    \begin{tabular}{lcccc}
        \toprule
        \multirow{2}{*}[-0.5\dimexpr \aboverulesep + \belowrulesep + \cmidrulewidth]{Model} & \multicolumn{2}{c}{Unconditional} & \multicolumn{2}{c}{Conditional} \\
        \cmidrule(lr){2-3} \cmidrule(lr){4-5}
        & FID$\downarrow$ & NFE$\downarrow$ & FID$\downarrow$ & NFE$\downarrow$ \\
        \midrule
        DDPM~\cite{ho2020denoising} & 3.17 & 1000     & - & - \\
        DDIM~\cite{song2021denoising} & 4.16 & 100     & - & - \\
        ScoreSDE~\cite{song2021scorebased} & 2.20 & 2000     & - & - \\
        iDDPM~\cite{nichol2021improved} & 2.90 & 1000     & - & - \\
        LSGM~\cite{vahdat2021score} & 2.10 & 138     & - & - \\
        CLD-SGM~\cite{dockhorn2022scorebased} & 2.25 & 312     & - & - \\
        STF~\cite{xu2022stable}    & 1.90 & 35     & - & -     \\
        ST~\cite{kim2022soft} & 2.33 & 2000     & - & - \\
        PFGM~\cite{xu2022poisson} & 2.35 & 110     & - & - \\
        INDM~\cite{kim2022maximum} & 2.28 & 2000     & - & - \\
        PFGM++~\cite{xu2023pfgm++} & 1.93 & 35     & - & - \\
        PSLD~\cite{pandey2023complete}    & 2.10 & 246     & - & -     \\
        ES~\cite{ning2024elucidating} & 1.95 & 35     & 1.80 & 35     \\
        \midrule
        \multirow{2}{*}{EDM (Heun)~\cite{karras2022elucidating}}  & 2.01 & 35     & 1.83 & 35     \\
                                                            & 2.03 & 65     & 1.90 & 89     \\
        \cmidrule{1-1}
        \multirow{2}{*}{EDM+DG~\cite{kim2023refining}}          & \underline{1.78} & 35     & \underline{1.66} & 35     \\
                                                            & 1.90 & 65     & 1.72 & 89     \\
        \cmidrule{1-1}
        \multirow{2}{*}{EDM+Restart~\cite{xu2023restart}}       & 1.95 & 43     & 1.85 & 43     \\
                                                            & 1.93 & 65     & 1.90 & 89     \\
        \cmidrule{1-1}
        \bf{EDM+DiffRS (ours)}  & \bf{1.59} & 64.06  & \bf{1.52} & 88.22  \\
        \bottomrule
    \end{tabular}
    }
    \label{tab:cifar10_SOTA_all}
\end{table}
\textbf{Experimental Setting}~
We primarily use the pre-trained networks on CIFAR-10 and ImageNet 64$\times$64 from EDM~\cite{karras2022elucidating}, which is known for the superior performance of the pre-trained models. For ImageNet 256$\times$256, we use the checkpoint from DiT~\cite{peebles2023scalable}. Additional results on other datasets (e.g., FFHQ~\cite{karras2019style}, AFHQv2~\cite{choi2020stargan}) and networks (e.g., DDPM++ cont.~\cite{song2021scorebased}) are provided in \cref{sec:app_add_exp}. All settings related to the discriminator are identical to \citet{kim2023refining}, which is provided in \cref{subsec:app_disc_set}. Note that the process of sampling from a pre-trained model and training a discriminator requires significantly less time and memory than training a diffusion model. Further experimental details are specified in \cref{sec:app_exp_set}. We mainly evaluate the generation performance using the Fréchet Inception Distance (FID)~\cite{heusel2017gans} on 50K samples, and we report the number of function evaluations (NFE) on the diffusion network. In the case of DiffRS, the NFE varies for each sample, so we take the average NFE of the samples. 

\begin{table}[t]
    \centering
    \caption{Performance comparison on class-conditional ImageNet 64$\times$64. The values in the first block are from the original paper.}
    \vskip 0.1in
    \adjustbox{max width=\linewidth}{%
    \begin{tabular}{lcc}
        \toprule
        Model & FID$\downarrow$ & NFE$\downarrow$    \\
        \midrule
        DDPM~\cite{ho2020denoising} & 11.0 &  250     \\
        iDDPM~\cite{nichol2021improved} & 2.92 &  250     \\
        ADM~\cite{dhariwal2021diffusion} & 2.07 & 250      \\
        CFG~\cite{ho2021classifierfree} & 1.55 & 250      \\
        CDM~\cite{ho2022cascaded} & 1.48 & 8000      \\
        RIN~\cite{jabri2022scalable} & \bf{1.23} & 1000      \\
        VDM++~\cite{kingma2023understanding} & 1.43 & 511      \\
        \midrule
        EDM (Heun)~\cite{karras2022elucidating} & 2.18 & 511      \\
        EDM (SDE)~\cite{karras2022elucidating} & 1.38 & 511      \\
        EDM+DG~\cite{kim2023refining} & 1.38 & 511      \\
        EDM+Restart~\cite{xu2023restart} & 1.37 & 623 \\
        \bf{EDM+DiffRS (ours)} & \underline{1.26} & 273.93      \\
        \bottomrule
    \end{tabular}
    }
    \label{tab:imagenet64_sota}
\end{table}

\begin{figure}
    \centering
    \includegraphics[width=\linewidth]{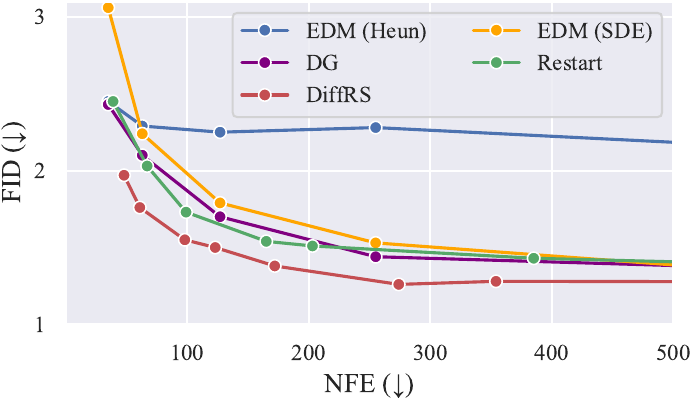}
    \caption{FID vs. NFE on ImageNet 64$\times$64 with EDM.}
    \label{fig:nfe_fid_imagenet}
\end{figure}

\begin{table}[t]
    \centering
    \caption{Performance comparison on class-conditional ImageNet 256$\times$256 with DiT-XL/2-G~\cite{peebles2023scalable}. `Time' is the average sampling time to generate 100 samples in minutes.}
    \vskip 0.1in
    \adjustbox{max width=\linewidth}{%
    \begin{tabular}{lcccccccc}
        \toprule
        Sampler & NFE$\downarrow$ & Time$\downarrow$& FID$\downarrow$ & sFID$\downarrow$ & IS$\uparrow$ & Prec$\uparrow$ & Rec$\uparrow$ &F1$\uparrow$    \\
        \midrule
        \multirow{3}{*}{\shortstack[l]{DDPM\\ \cite{ho2020denoising}}}    & 250 & 3.71 & 2.30 & 4.72 & 277.2 & 0.826 & 0.579 & 0.681 \\
                & 300 & 4.38 & 2.33 & 4.69 & 280.8 & 0.830 & 0.582 & 0.684 \\
                & 415 & 5.91 & 2.30 & \bf{4.68} & 279.8 & 0.831 & 0.572 & 0.678 \\
        \midrule
        \multirow{3}{*}{\shortstack[l]{DG\\ \cite{kim2023refining}}}      & 250 & 4.02 & 1.88 & 5.15 & 284.1 & 0.786 & 0.633 & 0.701 \\
                & 300 & 4.76 & 1.98 & 5.35 & \bf{287.9} & 0.793 & 0.621 & 0.696 \\
                & 375 & 5.87 & 1.83 & 4.99 & \bf{287.9} & 0.791 & 0.624 & 0.698 \\
        \midrule
        \bf{DG+DiffRS (ours)} & 306.88 & 5.87 & \bf{1.76} & \bf{4.68} & 279.1 & 0.796 & 0.629 & \bf{0.703}  \\
        \bottomrule
    \end{tabular}
    }
    \label{tab:imagenet256_sota}
\end{table}

\begin{figure}[t]
     \centering
     \begin{subfigure}[b]{0.32\linewidth}
         \centering
         \includegraphics[width=\linewidth]{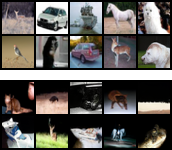}
         \caption{$\sigma=28.4$}
         \label{subfig:image3}
     \end{subfigure}
     \hfill
     \begin{subfigure}[b]{0.32\linewidth}
         \centering
         \includegraphics[width=\linewidth]{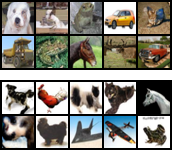}
         \caption{$\sigma=1.92$}
         \label{subfig:image9}
     \end{subfigure}
     \hfill
     \begin{subfigure}[b]{0.32\linewidth}
         \centering
         \includegraphics[width=\linewidth]{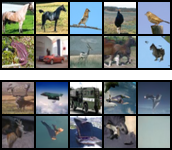}
         \caption{$\sigma=0.002$}
         \label{subfig:image17}
     \end{subfigure}
        \caption{Generated images with the highest (top) and lowest (bottom) acceptance probability at each timestep, obtained using the EDM (Heun) sampler on CIFAR-10. $\sigma=\{28.4, 1.92, 0.002\}$ corresponds to the $t=\{15, 9, 1\}$, respectively, with $T=18$.}
        \label{fig:image}
\end{figure}

\subsection{Analysis on Benchmark Datasets}
\label{sec:bench}

\textbf{CIFAR-10}~
\cref{tab:cifar10_SOTA_all} presents the performance of previous diffusion models and our proposed method on CIFAR-10. The proposed method achieves new SOTA with FID scores of 1.59 for the unconditional case and 1.52 for the class-conditional case. 

For a detailed analysis, the second block of \cref{tab:cifar10_SOTA_all} compares samplers that improve the sampling process using the same fixed pre-trained diffusion model on CIFAR-10. DiffRS exhibits the best performance under the same diffusion checkpoint. DiffRS is based on Heun's 2\textsuperscript{nd} order sampler (Heun) with 35 NFEs, and the NFE is increased due to rejection. For a fair comparison, we evaluate the baseline samplers with the same NFEs as DiffRS, and we observe that DiffRS still outperforms other baseline samplers under the same NFEs.

\textbf{ImageNet 64$\times$64}~
\cref{tab:imagenet64_sota} shows the performance for class-conditional image generation on ImageNet 64$\times$64. We report the best FID performances over NFE for each method. DiffRS achieves competitive performance on class-conditional ImageNet 64$\times$64, approaching SOTA with an FID score of 1.26 while requiring fewer NFEs compared to the current SOTA model (1.23 with 1000 NFEs).

In \cref{fig:nfe_fid_imagenet}, we evaluate the FID values of various NFEs for each method with the fixed pre-trained diffusion checkpoint on ImageNet 64$\times$64. We compare with the deterministic sampler (Heun) and the stochastic sampler (SDE) proposed by EDM~\cite{karras2022elucidating}, DG~\cite{kim2023refining}, and Restart~\cite{xu2023restart}. DG and DiffRS utilize Heun as the base sampler for small NFE regime and switch to the SDE sampler for large NFE regime. Restart employs Heun as the base sampler because the method is inherently based on the ODE sampler. DiffRS adjusted the backbone sampler and the value of $\gamma$ to measure performance on different NFEs, as detailed in \cref{subsec:app_config_diffrs}. Notably, DiffRS consistently outperforms on all NFE regimes. We include the uncurated generated images in \cref{subsec:app_gen_img}. These results highlight the effective and efficient sampling capabilities of DiffRS from the provided pre-trained network information.

\textbf{ImageNet 256$\times$256}~
We perform the experiment on high-resolution class-conditional image generation using ImageNet 256$\times$256 with DiT-XL/2-G~\cite{peebles2023scalable}. We apply DiffRS to DG sampler, and we also measure the performances of DDPM and DG on comparable NFEs and sampling time. As shown in \cref{tab:imagenet256_sota}, DiffRS performs better than DDPM and DG on the FID metric. Additionally, DiffRS achieves performance on par with the best results for the sFID and F1 metrics, while DDPM and DG have lower performance on one of these metrics. Therefore, DiffRS can be effectively used for sample refinement in high-resolution image generation.

\textbf{Acceptance Probability}~
\cref{fig:image} visualizes the top 10 and bottom 10 images for each timestep, determined by calculating the acceptance probability for 50,000 generated CIFAR-10 images sampled by the EDM (Heun) sampler. We observe that the top images have better visual quality. Conversely, for the bottom images, the images at large timesteps often have an overall unclear appearance, and the images at small timesteps have distortions in finer details. DiffRS effectively eliminates these problematic images, resulting in new high-quality images.

\begin{table}[t]
    \centering
    \caption{Ablation studies on unconditional CIFAR-10.}
    \vskip 0.1in
    \adjustbox{max width=\linewidth}{%
    \begin{tabular}{r@{\hspace{\tabcolsep}}lcc}
        \toprule
         & Methods & FID$\downarrow$ & NFE$\downarrow$    \\
        \midrule
        & No rejection sampling & 2.01 & 35 \\
        \midrule
        $\mathtt{(a)}$ & No sequential rejection sampling & 3.73 & 295.34 \\
        $\mathtt{(b)}$ & Marginal sequential rejection sampling & 1.66 & 63.57 \\
        \midrule
        $\mathtt{(c)}$ & Re-init. to $t+1$ by one-step forward only & 1.84 & 47.69 \\
        $\mathtt{(d)}$ & Re-init. to $T$ by prior distribution & 1.72 & 138.07 \\
        \midrule
        & DiffRS & \bf{1.59} & 64.06      \\
        \bottomrule
    \end{tabular}
    }
    \label{tab:abl}
\end{table}

\begin{figure*}[t]
    \centering
    \begin{minipage}{.32\linewidth}
        \centering
        \includegraphics[width=\linewidth]{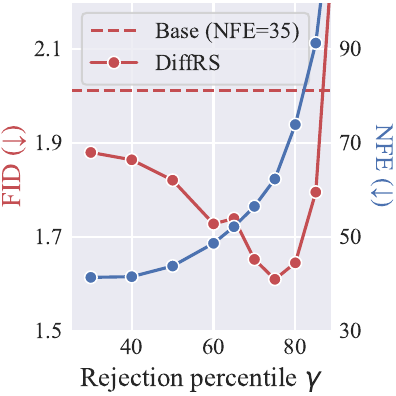}
        \caption{Sensitivity analysis of $\gamma$ on unconditional CIFAR-10.}
        \label{fig:gamma}
    \end{minipage}
    \hfill
    \begin{minipage}{.32\linewidth}
        \centering
        \includegraphics[width=\linewidth]{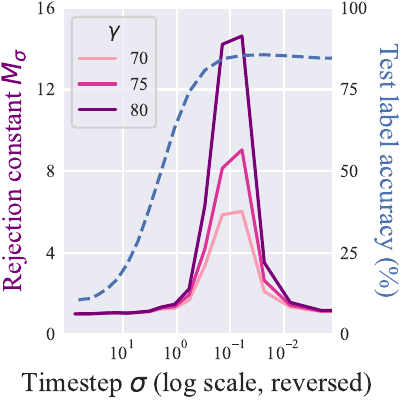}
        \caption{Rejection constant $M_\sigma$ over timesteps on unconditional CIFAR-10.}
        \label{fig:reject_const}
    \end{minipage}
    \hfill
    \begin{minipage}{.32\linewidth}
    \centering
    \includegraphics[width=\linewidth]{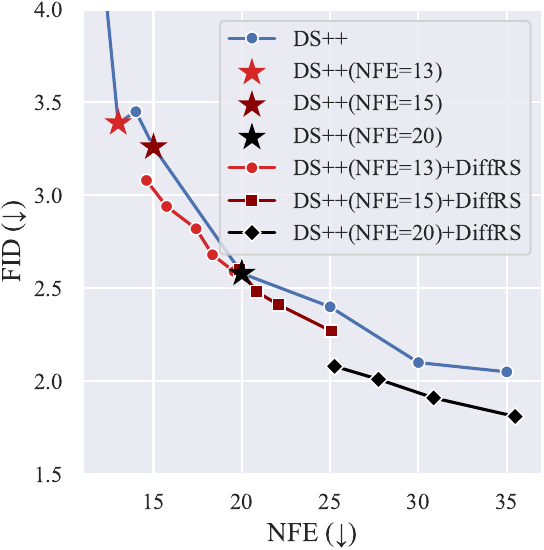}
    \caption{FID vs. NFE on unconditional CIFAR-10 with DPM-Solver++ (DS++).}
    \label{fig:nfe_fid_dpmsolver}
    \end{minipage}
\end{figure*}

\subsection{Ablation Studies}
\textbf{Sequential Rejection Sampling}~
We investigate the effect on the sequential rejection sampling based on the transition kernel, considering two scenarios: $\mathtt{(a)}$ marginal rejection sampling only at $t=0$ using $L_0$, and $\mathtt{(b)}$ sequential rejection sampling based on the marginal probability using $L_t$. As seen in $\mathtt{(a)}$ of \cref{tab:abl}, performance deteriorates without sequential rejection sampling, attributed to the challenges of density ratio estimation in high-dimensional data space~\cite{rhodes2020telescoping}. Additionally, rejections require iterative sampling from the prior distribution, significantly increasing the NFE.
In contrast, DiffRS performs sequential rejection sampling utilizing the time-dependent density ratios. As $t$ increases, the two distributions in the ratio become closer, leading to relatively accurate ratio estimation~\cite{kim2024training}. Moreover, rejections at intermediate timesteps contribute to a relative reduction in NFE.
On the other hand, in case of $\mathtt{(b)}$, using the marginal probability for sequential rejection sampling leads to performance degradation due to the mismatch between sampling and proposal distributions.

\textbf{Re-initialization}~
The third block in \cref{tab:abl} presents the variants of the re-initialization methods. In the case of rejection at timestep $t$, the first method, denoted $\mathtt{(c)}$, performs only one-step forward to $t+1$ and continues DiffRS from $t+1$; and the second method, denoted $\mathtt{(d)}$, transitions to timestep $T$ and restarts DiffRS from the prior distribution. The results show that both variants outperform the backbone sampler, but fall short of the performance of the proposed re-initialization method. In the case of $\mathtt{(c)}$, the re-initialized samples could deviate from the true distribution $q_{t+1}$, leading to a drop in performance. In the case of $\mathtt{(d)}$, there is a significant increase in NFE because sampling is restarted from timestep $T$. In contrast, our re-initialization method performs additional rejection sampling on the samples obtained through the forward step, attempting to initialize similar to the true distribution. Furthermore, by conducting the adequate number of forward steps for each sample, our method achieves superior performance at suitable NFEs.

\textbf{Rejection Constant}~
\cref{fig:gamma} shows the effect of the hyperparameter $\gamma$ on FID and NFE on CIFAR-10, where the rejection percentile $\gamma$ determines the rejection constants $M_t $ in the experiment. We observe that the NFE increases exponentially with increasing $\gamma$. While DiffRS generally has a better FID than the base sampler, there is an increase beyond an extreme threshold of $\gamma$. We empirically observe that the FID tends to increase when the NFE exceeds 2-3 times that of the base sampler. Therefore, we set $\gamma$ to keep the NFE at this level, typically in the range of $[75, 85]$.

\cref{fig:reject_const} visualizes the rejection constant $M_\sigma$ over timesteps on unconditional CIFAR-10 under various rejection percentile $\gamma$. As the rejection constant is inversely proportional to the acceptance probability~\cite{ripley2009stochastic}, a higher rejection constant implies a higher proportion of rejected samples. The distribution of the rejection constant over timesteps is bell-shaped, with a peak around $\sigma=0.1$.
Interestingly, Restart~\cite{xu2023restart} also adds noise around this timestep, which was chosen heuristically. 

To further analyze this interval, we include the test label accuracy of a time-dependent classifier trained by CIFAR-10 (blue dotted line). This result indicates the level of semantic information in the images at each timestep. We observe that the sample quality becomes distinguishable once a certain level of semantic information is reached. Also, in regions very close to the data space, the rejection rate decreases as the sample quality is almost determined.


\subsection{Application to Fast Sampler}

Diffusion models inherently suffer from problems of sampling speed due to the need for iterative sampling. To address this, various methods for fast sampling, such as the use of efficient ODE and SDE solvers, have been proposed~\cite{jolicoeur2021gotta,lu2022dpm,Dockhorn2022genie,zhang2023fast}. Most of these methods aim to follow the perturbed data distribution $q_t(\rvx_t)$ at time $t$, making it possible to apply DiffRS to these fast samplers.

We verify this experimentally on unconditional CIFAR-10 with DPM-Solver++, one of the few-step accelerated sampling methods~\cite{lu2022dpm,lu2022dpm++}. As shown in \cref{fig:nfe_fid_dpmsolver}, when comparing stars and line segments of the same color, we observe that although additional NFEs are incurred, the performance is improved compared to the base sampler. Additionally, we find that the performance is improved compared to the same NFEs of DPM-Solver++.

\subsection{Application to Distillation Methods}

Diffusion distillation methods are an alternative approach to accelerating the sampling process. They aim to obtain a distilled generative model with fewer NFEs from the information of the existing diffusion model process~\cite{salimans2022progressive,song2023consistency,meng2023distillation}. As discussed in \cref{subsec:overall}, DiffRS can be applied to diffusion distillation methods where an intermediate sample $\rvx_t$ is required to follow a perturbed data distribution $q_t(\rvx_t)$.

To investigate the effectiveness of DiffRS in distillation methods, we apply it to the Consistency Distillation (CD)~\cite{song2023consistency}. We use CD with 2 and 7 NFEs as base samplers. 
For DiffRS, we adjust the hyperparameter $\gamma$ to observe the changes in FID over NFE. \cref{fig:nfe_fid_imagenet_cd} shows that the combination of CD and DiffRS can generate images with an FID of less than 3.0 at an NFE nearly 10. This result suggests that DiffRS can also be effectively applied to diffusion distillation models.

\begin{figure}
    \centering
    \begin{minipage}{.48\linewidth}
            \centering
            \includegraphics[width=\linewidth]{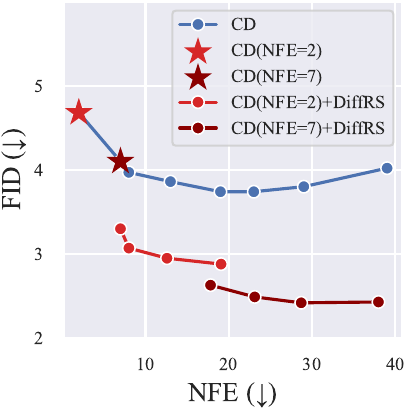}
            \caption{FID vs. NFE on ImageNet 64$\times$64 with CD.}
            \label{fig:nfe_fid_imagenet_cd}
    \end{minipage}%
    \hfill
    \begin{minipage}{.48\linewidth}
            \centering
            \includegraphics[width=\linewidth]{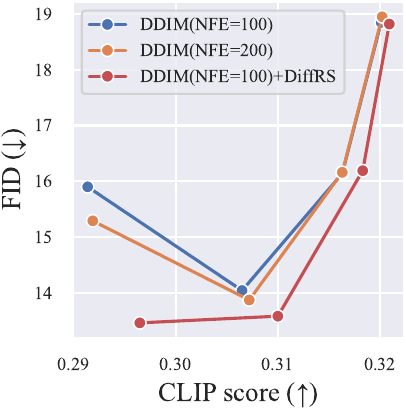}
            \caption{FID vs. CLIP score with Stable Diffusion v1.5.}
            \label{fig:fid_clip_sd}
    \end{minipage}
\end{figure}

\subsection{Application to Large-scale Text-conditional Model}

We further show that DiffRS can be applied to large-scale text-conditional diffusion models such as Stable Diffusion~\cite{rombach2022high}. We use the publicly available Stable Diffusion v1.5 pre-trained on LAION-5B~\cite{schuhmann2022laion} with a resolution of 512$\times$512. We apply DiffRS to DDIM~\cite{song2021denoising} with 100 NFEs. Following the evaluation protocol of previous studies~\cite{nichol2022glide,xu2023restart}, we generate 5,000 images from captions randomly sampled from the COCO~\cite{lin2014microsoft} validation set using the classifier-free guidance method~\cite{ho2021classifierfree}. We evaluate the sample quality using the FID metric and measure the image-text alignment through the CLIP score~\cite{hessel2021clipscore}.

\cref{fig:fid_clip_sd} plots the trade-off between FID and CLIP scores, varying the classifier-free guidance weights. DiffRS exhibits a superior FID for the same CLIP score, with an average of 166 NFEs. In contrast, the performance of DDIM did not significantly improve even with an increased number of NFEs. \cref{fig:example_sd} visualizes the example of images generated by DDIM and ours.
These results demonstrate the scalability of our model to effectively improve the sampling performance of a well-trained diffusion model even in text-to-image generation scenarios.

\begin{figure}[t]
    \centering
    \includegraphics[width=0.98\linewidth]{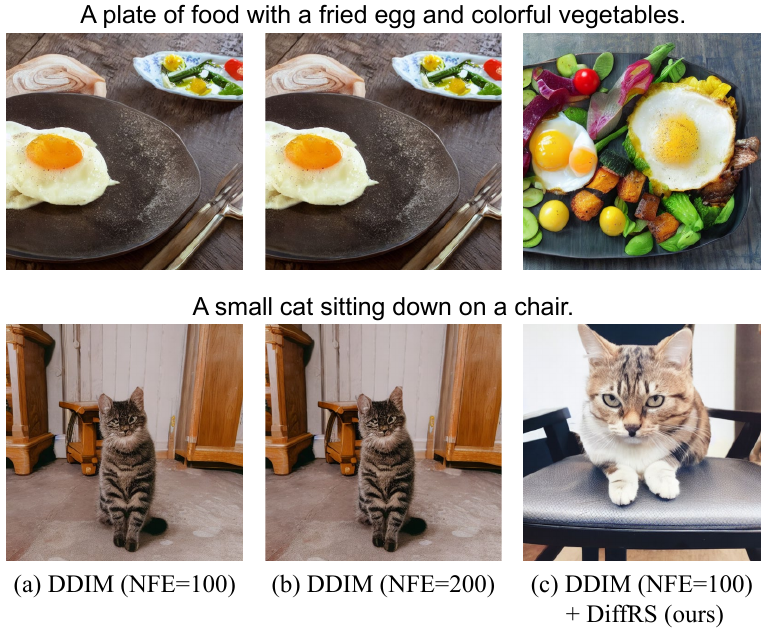}
    \caption{Example of generated images with Stable Diffusion v1.5. We use a classifier-free guidance weight of 2, and images on the same row are generated from the same noise from the prior distribution and the text prompt located above.}
    \label{fig:example_sd}
\end{figure}

\section{Conclusion}
\label{sec:conc}


We present Diffusion Rejection Sampling (DiffRS), a new diffusion sampling approach that ensures alignment between the reverse transition and the true transition at each timestep. The acceptance probability is estimated by training a time-dependent discriminator. We also propose the re-initialization method for DiffRS to effectively and efficiently refine the rejected samples. Theoretical analysis shows that discriminator training tightens the upper bound on the divergence between the data distribution and the refined distribution by DiffRS. Empirically, DiffRS achieves the state-of-the-art performances on the benchmark datasets, and DiffRS demonstrates its effectiveness on few-step accelerated samplers, diffusion distillation models, and large-scale text-to-image generation models.

Potential future work includes applying advanced sampling methods, such as Metropolis-Hastings sampling~\cite{turner2019metropolis}, to diffusion models. Additionally, developing methods to deal with discrepancies between the data distribution learned by a pre-trained diffusion model and the target data distribution, such as focusing on minority samples or the presence of label noise~\cite{um2024dont,na2024labelnoise}, will be promising applications of DiffRS.




\section*{Acknowledgements}

This research was supported by AI Technology Development for Commonsense Extraction, Reasoning, and Inference from Heterogeneous Data (IITP) funded by the Ministry of Science and ICT (2022-0-00077).

\section*{Impact Statement}
This paper primarily focuses on improving sample quality and efficiency in the diffusion generation process. The application of our method is promising in various fields such as art, design, and entertainment. However, ethical considerations, including the responsible use of AI-generated content and the prevention of harmful information creation, require careful attention. Implementing features such as a safety checker module and invisible watermarking can address some of these concerns.


\bibliography{main}
\bibliographystyle{icml2024}

\newpage
\appendix
\onecolumn
\section{Proof of Theoretical Analysis}
\label{sec:app_proof}

In this section we provide a proof of \cref{thm:main}.

\thma*

\begin{proof}
    First, we provide the derivation of \cref{eq:elbo}, the upper bound of KL divergence between the data distribution and the model distribution, comes from \cite{ho2020denoising}.
    \begin{align}
    \label{eq:elbo_proof}
        D_{\text{KL}} (q_0|| p_0^{\boldsymbol{\theta}}) & = \mathbb{E}_{q_0} [ - \log  p_0^{\boldsymbol{\theta}} (\rvx_0)] - H(q_0) \\
        & \leq \mathbb{E}_{q_{0:T}} \Big [ - \log \frac{p_{0:T}^{\boldsymbol{\theta}}(\rvx_{0:T})}{q_{1:T|0}(\rvx_{1:T}|\rvx_0)} \Big ] - H(q_0) \\
        & = \mathbb{E}_{q_{0:T}} \Big [ - \log p_T^{\boldsymbol{\theta}} (\rvx_T) - \sum_{t=0}^{T-1} \log \frac{p_{t|t+1}^{\boldsymbol{\theta}}(\rvx_{t}|\rvx_{t+1})}{q_{t+1|t}(\rvx_{t+1}|\rvx_t)} \Big ] - H(q_0) \\
        & = \mathbb{E}_{q_{0:T}} \Big [ - \log p_T^{\boldsymbol{\theta}} (\rvx_T) - \sum_{t=0}^{T-1} \log \frac{p_{t|t+1}^{\boldsymbol{\theta}}(\rvx_{t}|\rvx_{t+1})}{q_{t|t+1}(\rvx_{t}|\rvx_{t+1})} \frac{q_t(\rvx_t)}{q_{t+1}(\rvx_{t+1})} \Big ] - H(q_0) \\
        & = \mathbb{E}_{q_{0:T}} \Big [ - \log \frac{p_T^{\boldsymbol{\theta}} (\rvx_T)}{q_T(\rvx_T)} - \sum_{t=0}^{T-1} \log \frac{p_{t|t+1}^{\boldsymbol{\theta}}(\rvx_{t}|\rvx_{t+1})}{q_{t|t+1}(\rvx_{t}|\rvx_{t+1})} - \log q_0(\rvx_0) \Big ] - H(q_0) \\
        & = D_{\text{KL}}(q_T|| p_T^{\boldsymbol{\theta}}) + \sum_{t=0}^{T-1} \mathbb{E}_{q_{t+1}} \Big [ D_{\text{KL}}(q_{t|t+1} || p_{t|t+1}^{\boldsymbol{\theta}}) \Big ] =: J(\boldsymbol{\theta}).
    \end{align}

    If we substitute $p_0^{\boldsymbol{\theta}}$ with $p_0^{\boldsymbol{\theta}, \boldsymbol{\phi}}$ in the above, the following equation for the refined distribution $p_0^{\boldsymbol{\theta}, \boldsymbol{\phi}}$ by DiffRS holds:
    \begin{align}
    \label{eq:refine_elbo_proof1}
        D_{\text{KL}} (q_0|| p_0^{\boldsymbol{\theta}, \boldsymbol{\phi}}) \leq D_{\text{KL}}(q_T|| p_T^{\boldsymbol{\theta},\boldsymbol{\phi}}) + \sum_{t=0}^{T-1} \mathbb{E}_{q_{t+1}} \Big [ D_{\text{KL}}(q_{t|t+1} || p_{t|t+1}^{\boldsymbol{\theta},\boldsymbol{\phi}}) \Big ].
    \end{align}

    By the relationship between $p^{\boldsymbol{\theta}}$ and $p^{\boldsymbol{\theta}, \boldsymbol{\phi}}$, as described in \cref{eq:refine1,eq:refine2}, each term in the upper bound is further derived as follows:
    \begin{align}
    \label{eq:refine_elbo_proof2}
        D_{\text{KL}}(q_T|| p_T^{\boldsymbol{\theta},\boldsymbol{\phi}}) & = \mathbb{E}_{q_T} [ - \log  {p_T^{\boldsymbol{\theta},\boldsymbol{\phi}} (\rvx_T)} + \log {q_T(\rvx_T)}] \\
        & = \mathbb{E}_{q_T} [ - \log  {p_T^{\boldsymbol{\theta}} (\rvx_T)} - \log {\bar{A}_T^{\boldsymbol{\phi}}(\rvx_T)} + \log {q_T(\rvx_T)}] \label{eq:refine_elbo_proof2_1} \\
        & = D_{\text{KL}}(q_T|| p_T^{\boldsymbol{\theta}}) + \mathbb{E}_{q_T} [ - \log {\bar{A}_T^{\boldsymbol{\phi}}(\rvx_T)} ] ,
    \end{align}
    \begin{align}
    \label{eq:refine_elbo_proof3}
        D_{\text{KL}}(q_{t|t+1} || p_{t|t+1}^{\boldsymbol{\theta},\boldsymbol{\phi}}) & = \mathbb{E}_{q_{t|t+1}} [ -\log p_{t|t+1}^{\boldsymbol{\theta},\boldsymbol{\phi}} (\rvx_t | \rvx_{t+1}) + \log q_{t|t+1} (\rvx_t | \rvx_{t+1}) ] \\
        & = \mathbb{E}_{q_{t|t+1}} [ -\log p_{t|t+1}^{\boldsymbol{\theta}}(\rvx_t|\rvx_{t+1}) - \log \bar{A}_t^{\boldsymbol{\phi}}(\rvx_t,\rvx_{t+1}) + \log q_{t|t+1} (\rvx_t | \rvx_{t+1}) ] \label{eq:refine_elbo_proof3_1}\\
        & = D_{\text{KL}}(q_{t|t+1} || p_{t|t+1}^{\boldsymbol{\theta}}) + \mathbb{E}_{q_{t|t+1}} [ - \log \bar{A}_t^{\boldsymbol{\phi}}(\rvx_t,\rvx_{t+1}) ]
    \end{align}

    Therefore, we can derive the upper bound of the KL divergence as follows:
    \begin{align}
    \label{eq:refine_elbo_proof4}
        D_{\text{KL}} (q_0|| p_0^{\boldsymbol{\theta}, \boldsymbol{\phi}}) & \leq D_{\text{KL}}(q_T|| p_T^{\boldsymbol{\theta}}) + \mathbb{E}_{q_T} [ - \log {\bar{A}_T^{\boldsymbol{\phi}}(\rvx_T)} ]  + \sum_{t=0}^{T-1} \mathbb{E}_{q_{t+1}} \Big [ D_{\text{KL}}(q_{t|t+1} || p_{t|t+1}^{\boldsymbol{\theta}}) + \mathbb{E}_{q_{t|t+1}} [ - \log \bar{A}_t^{\boldsymbol{\phi}}(\rvx_t,\rvx_{t+1}) ]\Big ] \\
        & = J(\boldsymbol{\theta}) + \mathbb{E}_{q_T} [ - \log {\bar{A}_T^{\boldsymbol{\phi}}(\rvx_T)} ] + \sum_{t=0}^{T-1} \mathbb{E}_{q_{t,t+1}}  [ - \log \bar{A}_t^{\boldsymbol{\phi}}(\rvx_t,\rvx_{t+1}) ] \\
        & = J(\boldsymbol{\theta}) + R(\boldsymbol{\phi}) =: J(\boldsymbol{\theta},\boldsymbol{\phi}). 
    \end{align}

    where $R(\boldsymbol{\phi}) := \mathbb{E}_{q_{T}} [ - \log \bar{A}_T^{\boldsymbol{\phi}} ] + \sum_{t=0}^{T-1} \mathbb{E}_{q_{t,t+1}} [ - \log \bar{A}_{t}^{\boldsymbol{\phi}} ]$.

    Moreover, the optimal discriminator $\phi^*$ satisfies that:
    \begin{align}
    \label{eq:refine_optimal}
        \bar{A}_T^{\boldsymbol{\phi^*}}(\rvx_T) = \frac{q_{T}(\rvx_T)}{p_{T}^{\boldsymbol{\theta}} (\rvx_T)}, \text{ and }
        \bar{A}_t^{\boldsymbol{\phi^*}}(\rvx_t,\rvx_{t+1}) = \frac{q_{t|t+1}(\rvx_t|\rvx_{t+1})}{p_{t|t+1}^{\boldsymbol{\theta}}(\rvx_{t}|\rvx_{t+1})}.
    \end{align}
    Substituting $\bar{A}_T^{\boldsymbol{\phi^*}}(\rvx_T)$ into \cref{eq:refine_elbo_proof2_1} and $\bar{A}_t^{\boldsymbol{\phi^*}}(\rvx_t,\rvx_{t+1})$ into \cref{eq:refine_elbo_proof3_1} respectively, we observe that each KL term becomes zero. Consequently, the upper bound $J(\boldsymbol{\theta},\boldsymbol{\phi})=0$, which leads to the KL divergence on the data space, $D_{\text{KL}} (q_0|| p_0^{\boldsymbol{\theta}, \boldsymbol{\phi}})$, to be zero.
\end{proof}

\section{Related Works}
\label{sec:app_rel}

\subsection{Reducing Sampling Error of Diffusion Models}
\label{subsec:app_diff}

The sampling error can be measured by the distribution discrepancy between the data distribution and the generated distribution. This error is decomposed into three factors: the network approximation error, the prior mismatch error, and the temporal-discretization error~\cite{kim2022maximum}. To reduce the temporal-discretization error, reducing the sampling interval, which increases the iterative sampling count, is a common strategy; but it comes at the cost of a higher number of network evaluations, which slows down the sampling speed.

A significant amount of research has focused on improving the expressiveness of diffusion models through advances in network architecture or objective structure. For example, some studies proposed loss weights for timesteps or regularization methods for the diffusion objectives~\cite{kim2022soft,kingma2023understanding,lai2023fp}. Additionally, alternative approaches involve the investigation of the effective latent space~\cite{vahdat2021score,rombach2022high,kim2022maximum}. Other efforts aim at learning an implicit prior distribution to minimize the prior mismatch error and reduce the sampling length~\cite{zheng2023truncated}. However, these methods require time-consuming training of the diffusion model.

\subsection{Rejection Sampling}
\label{subsec:app_rel_rej}

Several researches utilize rejection sampling to discard poor samples for better generation quality in generative models. \citet{grover2018variational} propose the rejection sampling on the approximated variational posterior of variational autoencoder. \citet{azadi2018discriminator} introduce the rejection sampling by utilizing the discriminator of the generative adversarial network (GAN) to adjust the implicit distribution of the GAN generator. Similarly, \citet{turner2019metropolis} combine the Metropolis-Hastings algorithm and GAN.
However, there is no previous attempt to improve the sampling quality of the diffusion model via rejection sampling. It should be noted that it is difficult to naively apply the rejection sampling to the diffusion model due to the nature of its iterative sampling process.

\section{Additional Experimental Settings}
\label{sec:app_exp_set}

\subsection{Configurations of Baseline Samplers}
\label{subsec:app_config_base}

We use the baseline samplers as follows: Heun's 2\textsuperscript{nd} ODE sampler (Heun)~\cite{karras2022elucidating}, Improved SDE sampler (SDE)~\cite{karras2022elucidating}, DG~\cite{kim2023refining}, and Restart~\cite{xu2023restart} for the standard benchmark datasets; and DDIM~\cite{song2021denoising} for the text-to-image generation task. 
We adopt the sampling hyperparameter settings  from the experiments of the original papers. In cases where the experiment was not performed in the original paper, we used settings as similar as possible.
In the CIFAR-10, FFHQ, and AFHQv2 experiments, we use the Heun sampler serves as the backbone sampler for DG and Restart.
In the ImageNet 64$\times$64 experiments, we use the better sampler between Heun and SDE at each NFE as the backbone sampler for DG, while we use Heun for Restart.
In the ImageNet 256$\times$256 experiments, we use the DDPM sampler as the backbone sampler for DG.
For DPM-Solver++~\cite{lu2022dpm++}, we apply the singlestep DPM-Solver++.
For the diffusion distillation method, we apply the multi-step consistency sampling for the consistency distillation model~\cite{song2023consistency}.

\subsection{Settings of Discriminator Training}
\label{subsec:app_disc_set}
We follow DG~\citep{kim2023refining} to train a time-dependent discriminator by utilizing the code and some checkpoints from the DG repositories.\footnote{\label{footnote:DG}\url{https://github.com/alsdudrla10/DG}}\footnote{\label{footnote:DG_imagenet}\url{https://github.com/alsdudrla10/DG_imagenet}} We use the provided checkpoints for CIFAR-10 and FFHQ generation, and we train our own discriminator for other datasets. 
Our discriminator is trained on a single NVIDIA GeForce RTX 4090 GPU using CUDA 11.8 and PyTorch 1.12 versions.
The discriminator structure consists of two stacked U-net encoders. The pre-trained U-net encoder is from ADM~\citep{dhariwal2021diffusion} utilized as a feature extractor.\footnote{\label{footnote:ADM}\url{https://github.com/openai/guided-diffusion}} We utilize a randomly initialized feature extractor for the COCO dataset and pre-trained extractor with ImageNet classification for the remaining dataset. The shallow U-net encoders are only the trainable parameters for discriminating, which maps from feature to logits. For the conditional diffusion backbones, the shallow U-net encoders are also designed as a conditional model. The specific configurations are described in \cref{tab:config_disc}.

\begin{table}[h!]
	\caption{Configurations of the discriminator.}
    \vskip 0.1in
	\label{tab:config_disc}
	\centering
    \adjustbox{max width=\linewidth}{%
	\begin{tabular}{lccccccccc}
		\toprule
		& \multicolumn{2}{c}{CIFAR-10 } & \multicolumn{2}{c}{ImageNet64 }&ImageNet256& FFHQ & AFHQv2& COCO \\
            \cmidrule(lr){2-3}\cmidrule(lr){4-5}\cmidrule(lr){6-6}\cmidrule(lr){7-7}\cmidrule(lr){8-8}\cmidrule(lr){9-9}
            \multicolumn{1}{l}{\textbf{Diffusion Backbone}}\\ 
            Model & EDM & EDM & EDM & CD & DiT-XL/2 & EDM&EDM& Stable Diffusion \\ 
            Conditional model & \ding{55} & \ding{52} & \ding{52}& \ding{52} & \ding{52}&\ding{55}&\ding{55}& \ding{52}\\ \midrule
		
            \multicolumn{1}{l}{\textbf{Feature Extractor}}\\
		Model & ADM & ADM & ADM & ADM&ADM&ADM&ADM&ADM \\ 
            Architecture & U-Net encoder & U-Net encoder & U-Net encoder & U-Net encoder &U-Net encoder&U-Net encoder&U-Net encoder&U-Net encoder\\ 
            Pre-trained & \ding{52} & \ding{52} & \ding{52} & \ding{52} & \ding{52}&\ding{52}&\ding{52}&\ding{55} \\ 
		Depth & 4 & 4 & 4 & 4& 4&4 &4&4\\
            Width & 128 &128 & 128 & 128&128&128&128&128 \\
		Attention Resolutions & 32,16,8 & 32,16,8 & 32,16,8 & 32,16,8&32,16,8&32,16,8&32,16,8 &32,16,8 \\
            Input shape (data)& (B,32,32,3) & (B,32,32,3) & (B,64,64,3) & (B,64,64,3)&(B,32,32,4)&(B,64,64,3)&(B,64,64,3)&(B,64,64,4) \\
            Output shape (feature)& (B,8,8,512) & (B,8,8,512) & (B,8,8,512) & (B,8,8,512) & (B,8,8,384)&(B,8,8,512)&(B,8,8,512)&(B,8,8,512) \\ \midrule

            \multicolumn{1}{l}{\textbf{Discriminator}}\\
		Model & ADM & ADM & ADM & ADM&ADM&ADM&ADM&ADM \\ 
            Architecture & U-Net encoder & U-Net encoder & U-Net encoder & U-Net encoder & U-Net encoder &U-Net encoder&U-Net encoder&U-Net encoder\\ 
            Pre-trained & \ding{52} & \ding{52} & \ding{55} & \ding{55} & \ding{55} &\ding{52}&\ding{55}&\ding{55} \\ 
		Depth & 2 & 2 & 2 & 2&2&4 &2&2\\
            Width & 128 &128 & 128 & 128&128&128&128&128 \\
		Attention Resolutions & 32,16,8 & 32,16,8 & 32,16,8 & 32,16,8 & 32,16,8&32,16,8&32,16,8&32,16,8 \\
            Input shape (feature)& (B,8,8,512) & (B,8,8,512) & (B,8,8,512) & (B,8,8,512)&(B,8,8,384)&(B,8,8,512)&(B,8,8,512)&(B,8,8,512) \\
            Output shape (logit)& (B,1) & (B,1) & (B,1) & (B,1) & (B,1)&(B,1)&(B,1)&(B,1) \\ \midrule
            
            \multicolumn{1}{l}{\textbf{Discriminator Training}}&&&&&&&\\
            Time scheduling & VP & VP & Cosine VP& Cosine VP&VP&Cosine VP& Cosine VP& Cosine VP \\
            Time sampling & Importance & Importance & Importance & Importance & Importance&Importance&Importance&Importance \\
            Time weighting & $\frac{g^2}{\sigma^2}$ & $\frac{g^2}{\sigma^2}$ & $\frac{g^2}{\sigma^2}$ & $\frac{g^2}{\sigma^2}$ & $\frac{g^2}{\sigma^2}$&$\frac{g^2}{\sigma^2}$&$\frac{g^2}{\sigma^2}$&$\frac{g^2}{\sigma^2}$ \\ 
            Batch size & 128 & 128 & 128 & 128 & 512 & 128 & 128 & 128 \\
            \# data samples &50,000 & 50,000 & 50,000 & 50,000 & 50,000&60,000&15,803&5,000 \\
            \# generated samples & 25,000 & 50,000 & 50,000& 50,000 & 50,000&60,000&15,803&5,000 \\
            \# Epoch & 60 & 250 & 20 & 50&20&250&20&10 \\ 
		\bottomrule
	\end{tabular}
    }
\end{table}

\begin{table}[ht!]
    \centering
    \caption{Configuration details for each experimental result.}
    \vskip 0.1in
    \adjustbox{max width=\linewidth}{%
    \begin{tabular}{cccccccc}
        \toprule
        \multicolumn{3}{c}{Pre-trained diffusion}  & \multicolumn{2}{c}{Performance} & \multicolumn{3}{c}{Configuration} \\
        \cmidrule(lr){1-3} \cmidrule(lr){4-5} \cmidrule(lr){6-8}
        Dataset & Task & Model & FID$\downarrow$ & NFE$\downarrow$ & Base sampler & Rejection percentile $\gamma$ & Max. iteration $K$   \\
        \midrule
        CIFAR-10 & Unconditional & DDPM++ cont. & 1.91 & 151.86 & EDM (SDE) (NFE=63) & 85 & $\infty$ \\ 
        \cmidrule(lr){3-8}
                        && EDM & 1.59 & 64.06 & EDM (Heun) (NFE=35) & 75 & 105 \\
        \\[-0.66em]
                           &&& 1.88 & 41.37 & EDM (Heun) (NFE=35) & 30 & $\infty$ \\
                           &&& 1.86 & 41.52 & EDM (Heun) (NFE=35) & 40 & $\infty$ \\
                           &&& 1.82 & 43.78 & EDM (Heun) (NFE=35) & 50 & $\infty$ \\
                           &&& 1.73 & 48.61 & EDM (Heun) (NFE=35) & 60 & $\infty$ \\
                           &&& 1.74 & 52.13 & EDM (Heun) (NFE=35) & 65 & $\infty$ \\
                           &&& 1.65 & 56.45 & EDM (Heun) (NFE=35) & 70 & $\infty$ \\
                           &&& 1.60 & 62.28 & EDM (Heun) (NFE=35) & 75 & $\infty$ \\
                           &&& 1.64 & 73.84 & EDM (Heun) (NFE=35) & 80 & $\infty$ \\
                           &&& 1.79 & 91.17 & EDM (Heun) (NFE=35) & 85 & $\infty$ \\
        \cmidrule(lr){3-8}
                        && EDM & 3.08 & 14.60 & DPM-Solver++ (NFE=13) & 20 & $\infty$ \\
                           &&& 2.94 & 15.74 & DPM-Solver++ (NFE=13) & 30 & $\infty$ \\
                           &&& 2.82 & 17.41 & DPM-Solver++ (NFE=13) & 40 & $\infty$ \\
                           &&& 2.68 & 18.34 & DPM-Solver++ (NFE=13) & 50 & $\infty$ \\
                           &&& 2.59 & 19.56 & DPM-Solver++ (NFE=13) & 60 & $\infty$ \\
                           &&& 2.60 & 19.88 & DPM-Solver++ (NFE=15) & 40 & $\infty$ \\
                           &&& 2.48 & 20.85 & DPM-Solver++ (NFE=15) & 45 & $\infty$ \\
                           &&& 2.41 & 22.10 & DPM-Solver++ (NFE=15) & 50 & $\infty$ \\
                           &&& 2.27 & 25.06 & DPM-Solver++ (NFE=15) & 60 & $\infty$ \\
                           &&& 2.08 & 25.25 & DPM-Solver++ (NFE=20) & 40 & $\infty$ \\
                           &&& 2.01 & 27.74 & DPM-Solver++ (NFE=20) & 50 & $\infty$ \\
                           &&& 1.91 & 30.86 & DPM-Solver++ (NFE=20) & 60 & $\infty$ \\
                           &&& 1.81 & 35.48 & DPM-Solver++ (NFE=20) & 70 & $\infty$ \\
        
        \midrule
        CIFAR-10 & Class-conditional & EDM & 1.52 & 88.22 & EDM (Heun) (NFE=35) & 80 & 105 \\
        \midrule
        FFHQ & Unconditional & EDM & 1.60 & 198.65 & EDM (Heun) (NFE=71)  & 90 & 213\\
        \midrule
        AFHQv2 & Unconditional & EDM & 1.80 & 144.92 & EDM (Heun) (NFE=71)  & 85 & 213\\
        \midrule
        ImageNet 64$\times$64 & Class-conditional & EDM & 1.97 & 48.23 & EDM (Heun) (NFE=27) & 60 & $\infty$ \\
                              &&& 1.76 & 60.95 & EDM (Heun) (NFE=27) & 70 & $\infty$ \\
                              &&& 1.55 & 98.60 & EDM (Heun) (NFE=27) & 80 & $\infty$ \\
                              &&& 1.50 & 123.00 & EDM (SDE) (NFE=63) & 60 & $\infty$ \\
                              &&& 1.38 & 171.95 & EDM (SDE) (NFE=63) & 70 & $\infty$ \\
                              &&& 1.26 & 273.93 & EDM (SDE) (NFE=127) & 70 & $\infty$ \\
                              &&& 1.27 & 353.60 & EDM (SDE) (NFE=127) & 75 & $\infty$ \\
                              &&& 1.27 & 1169.67 & EDM (SDE) (NFE=511) & 70 & $\infty$ \\
        \cmidrule(lr){3-8}
        && CD & 3.30 & 7.00 & CD (NFE=2) & 70 & $\infty$ \\
                              &&& 3.07 & 8.01 & CD (NFE=2) & 75 & $\infty$ \\
                              &&& 2.95 & 12.58 & CD (NFE=2) & 85 & $\infty$ \\
                              &&& 2.88 & 19.07 & CD (NFE=2) & 90 & $\infty$ \\
                              &&& 2.63 & 17.81 & CD (NFE=7) & 60 & $\infty$ \\
                              &&& 2.49 & 23.11 & CD (NFE=7) & 80 & $\infty$ \\
                              &&& 2.42 & 28.69 & CD (NFE=7) & 85 & $\infty$ \\
                              &&& 2.43 & 37.96 & CD (NFE=7) & 90 & $\infty$ \\
        \midrule
        ImageNet 256$\times$256 & Class-conditional & DiT-XL/2 & 1.76 & 306.88 & DDPM+DG (NFE=250) & 65 & $\infty$ \\
        \midrule
        COCO & Text-to-image & Stable Diffusion (weight=2) & 13.46 & 166.95 & DDIM (NFE=100) & 80 & $\infty$ \\
        && Stable Diffusion (weight=3) & 13.58 & 166.36 & DDIM (NFE=100) & 80 & $\infty$ \\
        && Stable Diffusion (weight=5) & 16.19 & 217.13 & DDIM (NFE=100) & 80 & $\infty$ \\
        && Stable Diffusion (weight=8) & 18.82 & 115.24 & DDIM (NFE=100) & 80 & $\infty$ \\
        \bottomrule
    \end{tabular}
    }
    \label{tab:exp_all}
\end{table}

\subsection{Configurations of DiffRS}
\label{subsec:app_config_diffrs}

We integrate DiffRS into the code implementation of each base sampler: DG codebase\footnoteref{footnote:DG} for EDM-based samplers; DG-ImageNet codebase\footnoteref{footnote:DG_imagenet} for ImageNet 256$\times$256; DPM-Solver-v3~\cite{zheng2023dpm} codebase\footnote{\label{footnote:DPM-Solver}\url{https://github.com/thu-ml/DPM-Solver-v3}} for DPM-Solver++; Consistency Models codebase\footnote{\label{footnote:CM}\url{https://github.com/openai/consistency_models}} for CD; and Restart codebase\footnote{\label{footnote:Restart}\url{https://github.com/Newbeeer/diffusion_restart_sampling}}, built on Diffusers\footnote{\url{https://github.com/huggingface/diffusers}}, for Stable Diffusion.
For the benchmark datasets, we utilize a single NVIDIA GeForce RTX 4090 GPU, CUDA 11.8, and PyTorch 1.12. 
For the text-to-image generation, we use a single NVIDIA L40S GPU with CUDA 11.8 and PyTorch 2.1. Our implementation is available at: \url{https://github.com/aailabkaist/DiffRS}.

To estimate the rejection constant $M_t$, we generate 1,000 samples with evaluating the unnormalized acceptance probability $\bar{A}^{\boldsymbol{\phi}}_t =\frac{ \hat{L}_t^{\boldsymbol{\phi}}(\rvx_t)}{ \hat{L}_{t+1}^{\boldsymbol{\phi}}(\rvx_{t+1})}$ using the trained discriminator. Then, we select the $\gamma$\textsuperscript{th} percentile values from these values as the rejection constant for each timestep, with the minimum value of $M_t$ set to one.
Additionally, we set a maximum iteration $K$ to prevent looping within a single path. If this limit is exceeded, we initialize the sampling again from the prior distribution. In most cases, we set $K$ to either $\infty$ or three times the NFE of the base sampler. The hyperparameters for each experiment, along with their corresponding performance, are provided in \cref{tab:exp_all}.

\subsection{Configurations of Pre-trained Diffusion Models}
\label{subsec:app_config_pretrained}

For CIFAR-10, we employ the pre-trained DDPM++ cont. and EDM models obtained from the EDM repository.\footnote{\label{footnote:edm}\url{https://github.com/NVlabs/edm}} For FFHQ and AFHQv2, we use the pre-trained EDM models also available in the EDM repository.\footnoteref{footnote:edm} In the case of ImageNet 64$\times$64, we use the pre-trained EDM model from the EDM repository\footnoteref{footnote:edm}, and the consistency distillation model from the Consistency Model repository.\footnoteref{footnote:CM} For ImageNet 256$\times$256, we use the pre-trained DiT-XL/2 from the DiT repository.\footnote{\label{footnote:DiT}\url{https://github.com/facebookresearch/DiT}} In the text-to-image generation task, we use Stable Diffusion v1.5 pre-trained on LAION-5B, available from HuggingFace.\footnote{\url{https://huggingface.co/runwayml/stable-diffusion-v1-5}}

\subsection{Evaluation Procedure}
\label{subsec:app_eval}

We evaluate the performance of diffusion models using Fréchet Inception Distance (FID). FID calculations are performed using the DG~\cite{kim2023refining} code, and we report the results for the random seeds.
For ImageNet 256$\times$256, we also report Inception Score (IS)~\cite{salimans2016improved}, sFID~\cite{nash2021generating}, Precision (Prec), Recall (Rec), and F1 of Prec and Rec~\cite{kynkaanniemi2019improved}, evaluated by ADM~\cite{dhariwal2021diffusion} code.
In the stable diffusion experiment, FID and CLIP score calculations are conducted using the Restart code. CLIP scores are evalated using the open-sourced ViT-g/14~\cite{ilharco2021openclip}.

\section{Additional Experiment Results}
\label{sec:app_add_exp}

\begin{figure}[t]
  \begin{minipage}[b]{0.63\linewidth}
  \begin{minipage}[t]{\linewidth}
    \centering
    \captionof{table}{Performance on FFHQ and AFHQv2 with EDM~\cite{karras2022elucidating}.}
    \vskip 0.1in
    \adjustbox{max width=0.8\linewidth}{%
    \begin{tabular}{lcccc}
        \toprule
        \multirow{2}{*}[-0.5\dimexpr \aboverulesep + \belowrulesep + \cmidrulewidth]{Sampler} & \multicolumn{2}{c}{FFHQ} & \multicolumn{2}{c}{AFHQv2} \\
        \cmidrule(lr){2-3} \cmidrule(lr){4-5}
        & FID & NFE & FID & NFE \\
        \midrule
        \multirow{2}{*}{EDM (Heun)~\cite{karras2022elucidating}}    & 2.41 & 71     & 2.00 & 71     \\
                & 2.43 & 199    & 2.05 & 145    \\
        \midrule
        \multirow{2}{*}{DG~\cite{kim2023refining}}      & 1.96 & 71     & 1.88 & 71     \\
                & 1.93 & 199    & 1.85 & 145     \\
        \midrule
        \bf{DiffRS (ours)}  & \bf{1.60} & 198.65  & \bf{1.80} & 144.92  \\
        \bottomrule
    \end{tabular}
    }
    \label{tab:ffhq}
  \end{minipage}
  \hfill
  \begin{minipage}[b]{\linewidth}
    \centering
        \captionof{table}{Performance on CIFAR-10 with DDPM++ cont.~\cite{song2021scorebased}.}
        \vskip 0.1in
        \begin{tabular}{lcc}
            \toprule
            Sampler & FID$\downarrow$ & NFE$\downarrow$    \\
            \midrule
            EDM (Heun)~\cite{karras2022elucidating} & 2.89 & 63 \\
            EDM (SDE)~\cite{karras2022elucidating} & 2.35 & 1023 \\
            Restart~\cite{xu2023restart} & 2.11 & 519 \\
            \bf{DiffRS (ours)} & 1.91 & 151.86     \\
            \bottomrule
        \end{tabular}
        \label{tab:exp_ddpm}
  \end{minipage}
  \end{minipage}
  \hfill 
  \begin{minipage}[t]{0.35\linewidth}
    \centering
    \includegraphics[width=\linewidth]{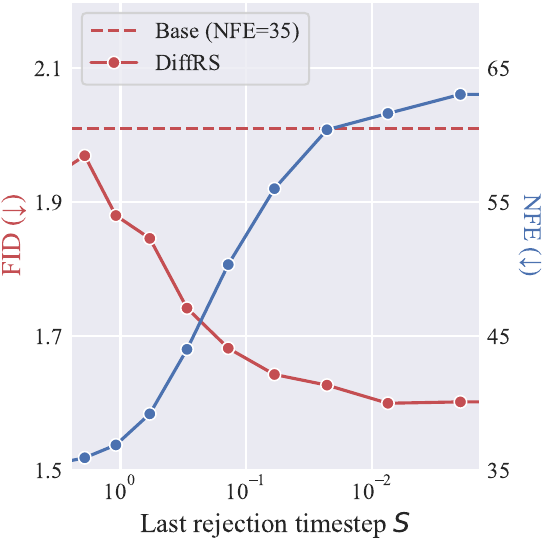}
    \caption{Trade-off between FID and NFE on unconditional CIFAR-10 varying the last rejection timestep.}
    \label{fig:laststep}
  \end{minipage}
\end{figure}

\subsection{Experimental Results on FFHQ and AFHQv2}
\label{subsec:app_ffhq}

In \cref{tab:ffhq}, we present the performance on FFHQ~\cite{karras2019style} and AFHQv2~\cite{choi2020stargan}. We use the Heun with 71 NFEs as the base sampler for DffRS and compare DiffRS to Heun and DG. Remarkably, DiffRS demonstrates significant improvements in FID over the base sampler on these benchmark datasets (+0.81 for FFHQ and +0.20 for AFHQv2). In addition, our method exhibits superior performance even with similar NFEs.

\subsection{Experimental Results on DDPM++ cont.}
\label{subsec:app_ddpm}

\cref{tab:exp_ddpm} shows the performance with the DDPM++ cont. model~\cite{song2021scorebased} on the unconditional CIFAR-10 dataset. The baseline results are taken from the reported performance of Restart~\cite{xu2023restart}. We find that DiffRS shows the superior performance. Therefore, DiffRS works effectively for other diffusion backbones as well.

\subsection{Additional Ablation Studies}
\label{subsec:app_abl}

\cref{fig:laststep} shows the changes in FID and NFE when DiffRS is applied only up to $S$ instead of applying it to all timesteps. As $S$ increases, indicating a smaller interval for applying rejection sampling, the FID degrades. Similar to the analysis of the rejection constant in the main manuscript (\cref{fig:reject_const}), a drastic change in FID and NFE is observed around $\sigma=0.1$.

\begin{figure*}[t]
    \centering
    \begin{minipage}{.25\linewidth}
        \centering
        \includegraphics[width=\linewidth]{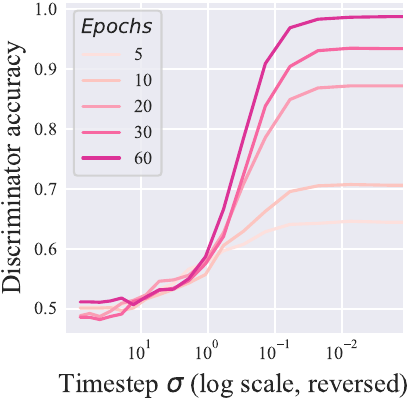}
        \caption{Accuracy of discriminator over each timestep varying discriminator training epochs, on unconditional CIFAR-10.}
        \label{fig:disc_acc}
    \end{minipage}
    \hfill
    \begin{minipage}{.25\linewidth}
        \centering
        \includegraphics[width=\linewidth]{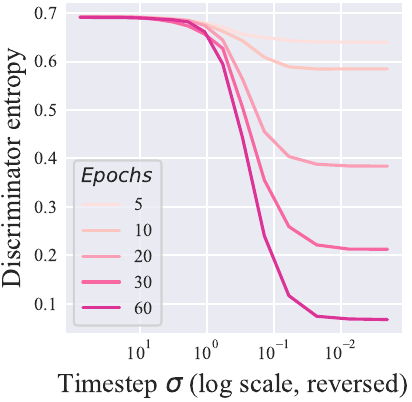}
        \caption{Entropy of discriminator over each timestep varying discriminator training epochs, on unconditional CIFAR-10.}
        \label{fig:disc_ent}
    \end{minipage}
    \hfill
    \begin{minipage}{.45\linewidth}
    \centering
    \includegraphics[width=\linewidth]{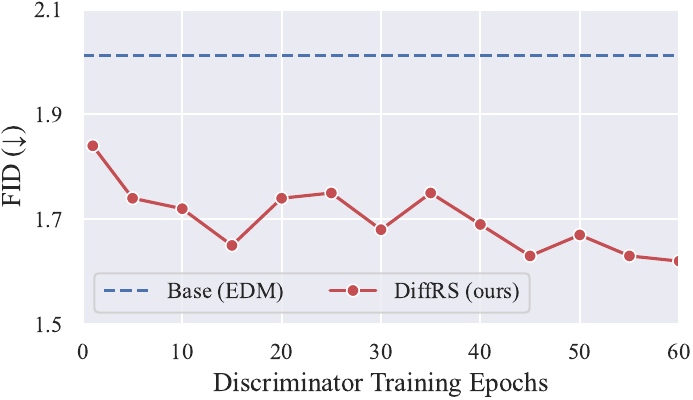}
    \caption{FID performance over discriminator training epochs on unconditional CIFAR-10.}
    \label{fig:disc_epochs}
    \end{minipage}
\end{figure*}

\subsection{Ablation Studies of Discriminator}
\label{subsec:app_abl_disc}

\textbf{Training Curve}~
As shown in \cref{fig:disc_acc,fig:disc_ent}, unlike GAN training, the discriminator training of our method is stable. This is because the score network that serves as the generator is pre-trained and fixed. Therefore, this is a single directional optimization process without min-max game, such as GAN. Experimentally, we plot the sample performance according to the number of discriminator training epochs. As shown in \cref{fig:disc_epochs}, we find that the performance improves and stabilizes already from the early epochs.

\begin{figure*}
    \centering
    \includegraphics[width=\linewidth]{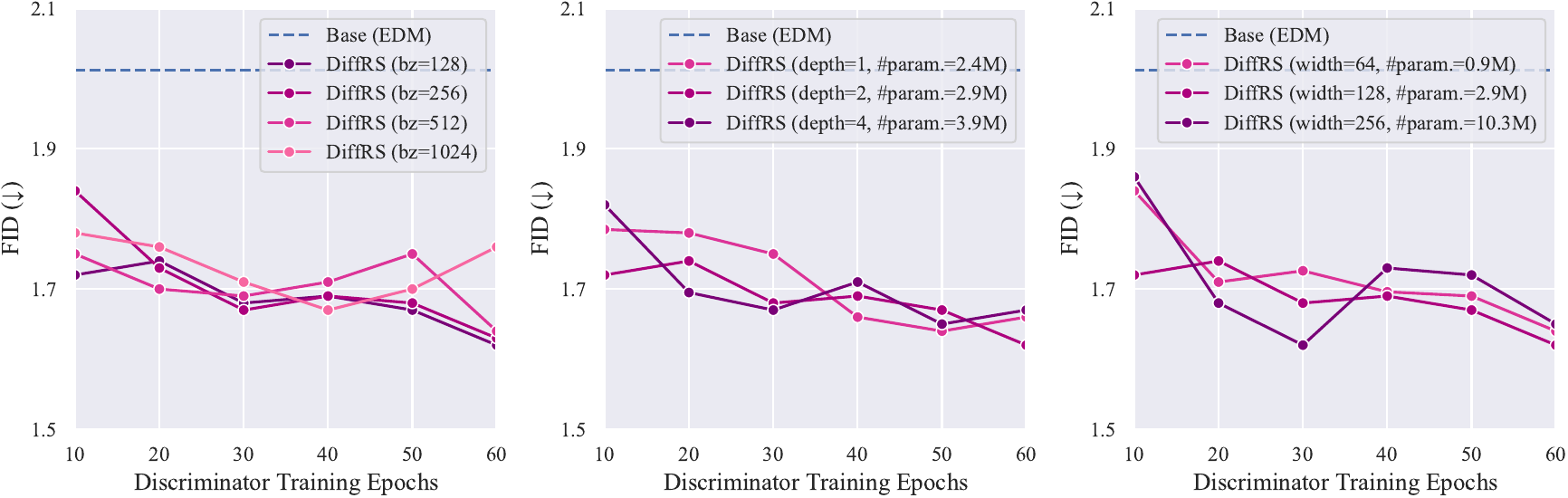}
    \caption{Ablation studies of the discriminator configurations on unconditional CIFAR-10. Each subfigure is for (top) batch size, (middle) depth of U-Net, and (bottom) width of U-Net. \texttt{bz} stands for batch size and \texttt{\#param.} is the number of discriminator parameters.}
    \label{fig:abl_disc}
\end{figure*}

\textbf{Configurations}~
We perform an ablation study to explore the effects of the discriminator configurations. We measured the sample performance across training epochs, varying the discriminator training batch size, and the depth and width of the U-Net. As shown in \cref{fig:abl_disc}, we observed superior performance compared to the base sampler across all settings.

\begin{figure*}
    \centering
    \includegraphics[width=0.8\linewidth]{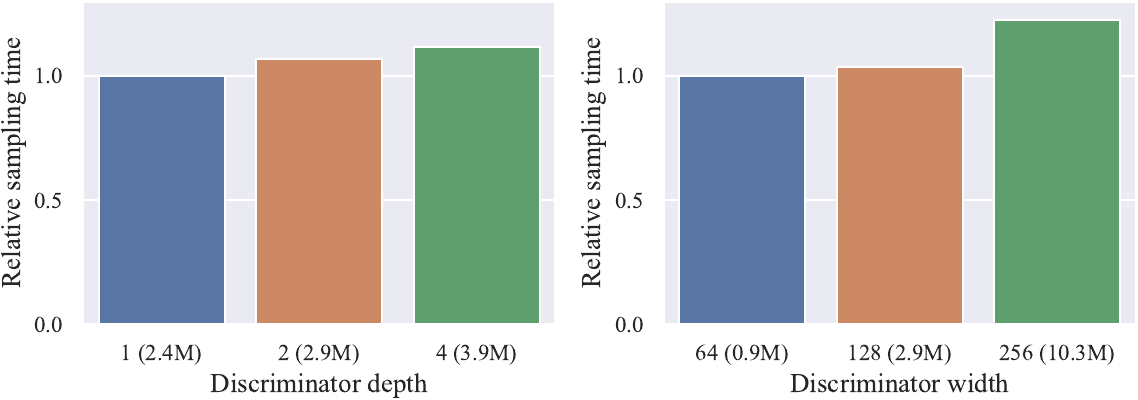}
    \caption{Relative sampling time varying discriminator configurations. The numbers in parentheses indicate the number of parameters in the discriminator.}
    \label{fig:abl_disc_time}
\end{figure*}

Also, we plot the sampling time according to the discriminator structure in \cref{fig:abl_disc_time}. As shown in the figure, although the sampling time increases slightly as the discriminator parameter size increases, the evaluation time of the diffusion models accounts for a larger proportion, resulting in a non-significant difference.

\begin{figure*}
    \centering
    \includegraphics[width=\linewidth]{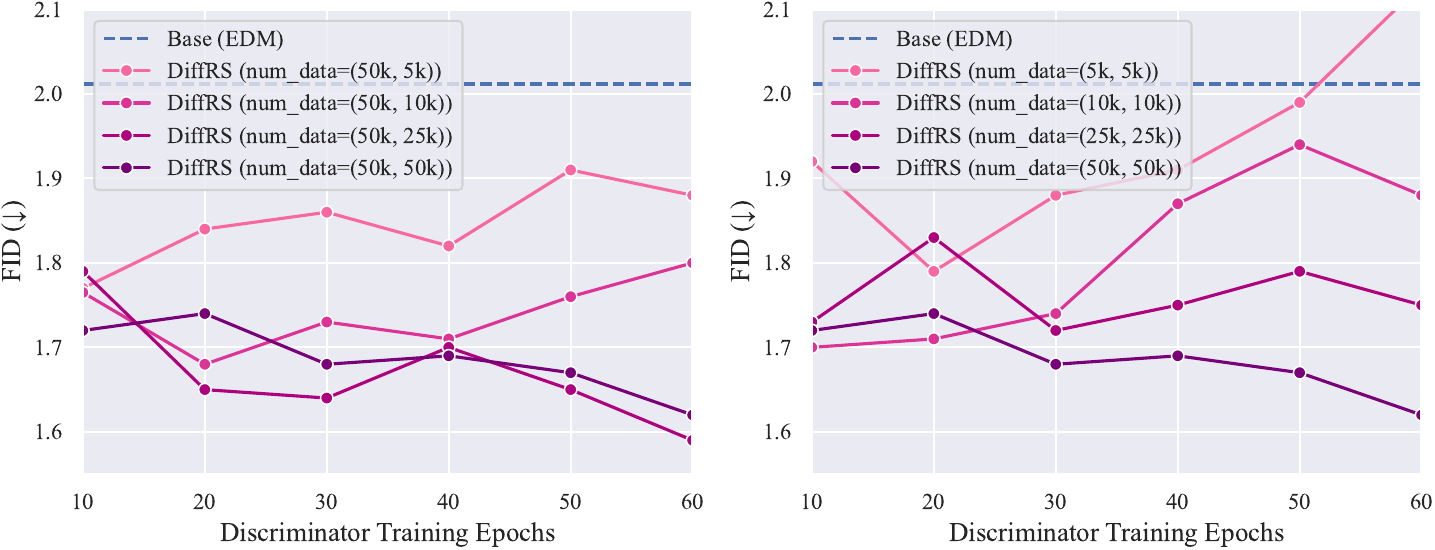}
    \caption{Ablation studies of the number of training samples for the discriminator on unconditional CIFAR-10. The tuple in the legend represents (number of training data, number of generated data).}
    \label{fig:abl_numdata}
\end{figure*}

\textbf{Number of Samples}~
We examine the performance according to the number of discriminator training data. We perform experiments on unconditional CIFAR-10 in two settings: 1) using all training images (50k examples) and varying the number of generated images, 2) using the same number of training images as generated images.

As shown in \cref{fig:abl_numdata}, when using all training images, even generating only 10\% of the training images (5k images) outperforms the base sampler. Also, we observe improved performance as the number of generated images increases. However, when matching the number of training images to the number of generated images, we observe a decrease in performance as training progresses when the number of samples is small. This also happens when all training images are used, but the number of generated images is small. We attribute this to the reduced number of training images leading to overfitting problems, resulting in inaccurate density ratio estimation by the discriminator.

Therefore, it is preferable to use all of the given training data, and more generated data generally improves performance. However, even using only 10\% of the training data can provide better performance than the base sampler.

\subsection{Estimation of Rejection Constant $M_t$}
\label{subsec:app_rej_const}

Theoretically, $M_t$ should always be greater than the ratio of the target distribution to the proposal distribution for all instances. Therefore, a proper estimator of $M_t$ would be the maximum value of the density ratio extracted from the samples, i.e., the rejection percentile $\gamma=100$(\%). However, in most our experiments, adjusting $\gamma$ in the range of [75, 85] worked well. Specifically, \cref{fig:gamma} in the main manuscript illustrates the performance variation with respect to $\gamma$, where it can be observed that the FID increases significantly as $\gamma$ becomes very large.

\begin{wrapfigure}[18]{R}{0.3\textwidth}
     \centering
     \includegraphics[width=\linewidth]{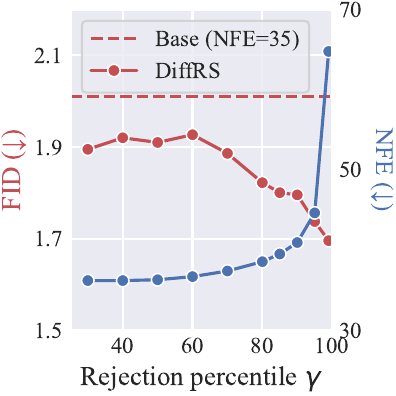}
     \caption{Sensitivity analysis of $\gamma$ with early-stage discriminator (trained by 5-epochs) on unconditional CIFAR-10.}
     \label{fig:gamma_sens}
\end{wrapfigure}

We believe that such cases are due to problems with the discriminator network used to estimate the density ratio. To investigate this, we measure the entropy and accuracy of the discriminator outputs of the training dataset for each timestep over discriminator training epochs. As shown in \cref{fig:disc_acc,fig:disc_ent}, for small epochs, both prediction confidence and accuracy are low, and as the epochs increase, confidence and accuracy increase significantly. This could indicate that overconfidence problems occur as training progresses, possibly leading to an inaccurate density ratio estimate that is skewed toward extreme values, thus degrading performance.

In this case, we believe that lowering the rejection constant $M_t$ by the rejection percentile $\gamma$ helped alleviate the problem of overconfidence in the discriminator. To investigate this further, we did a small experiment by limiting the training of discriminator to suppress the overconfidence problem. We examine the performance changes with respect to $\gamma$ for the early-stage discriminator (i.e., trained by 5-epochs). As shown in \cref{fig:gamma_sens}, the early-stage discriminator continues to perform better as $\gamma$ increases. While the performance is generally better than the baseline, it did not reach the best performance (FID=1.59) of the final discriminator (trained by 60-epochs). Therefore, while it is necessary to train the discriminator beyond a certain level, the overconfidence problem of neural networks can occur, but this can be mitigated by adjusting $\gamma$.

\begin{figure}[t]
    \begin{minipage}[t]{0.53\linewidth}
        \centering
        \captionof{table}{Sampling time (seconds) to generate 100 samples at 63 NFEs.}
        \vskip 0.1in
        \adjustbox{max width=\linewidth}{%
        \begin{tabular}{cccc}
            \toprule
            EDM (Heun) & EDM (SDE) & DG & DiffRS \\
            \midrule
            45.69 & 45.97 & 59.66 & 54.69 \\
            \bottomrule
        \end{tabular}
        }
        \label{tab:sampling_time}
    \end{minipage}
\hfill
    \begin{minipage}[t]{0.40\linewidth}
    \vspace{0pt}
        \centering
        \includegraphics[width=\linewidth]{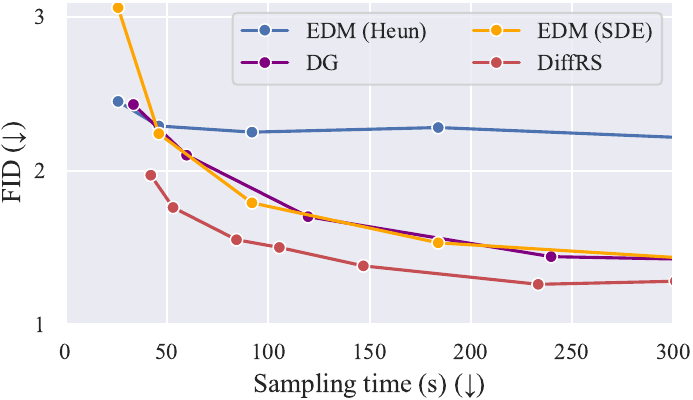}
        \caption{FID vs. Sampling time (seconds per 100 images) on ImageNet 64$\times$64 with EDM.}
        \label{fig:sampling_time}
    \end{minipage}
    \vspace{-1.7em}
\end{figure}

\subsection{Sampling Time}
\label{subsec:app_sampling_time}

DG and DiffRS require additional sampling time due to the use of an auxiliary discriminator network. DG requires discriminator evaluation and gradient computation at each timestep, while DiffRS only requires discriminator evaluation at each timestep. \cref{tab:sampling_time} shows the sampling time taken to generate 100 samples at the same NFE. DiffRS takes longer than the base samplers because of the discriminator evaluation, and DG takes more time due to the gradient computation. \cref{fig:sampling_time} illustrates the FID changes in terms of the sampling time required to generate 100 samples. As seen in the figure, our model demonstrates superior performance for the same sampling time.

\subsection{Generated Images}
\label{subsec:app_gen_img}

\cref{fig:images_cifar_uncond,fig:images_cifar_cond,fig:images_ffhq,fig:images_afhqv2} show the generated images of DiffRS on the benchmark datasets.
\cref{fig:images_imagenet_130,fig:images_imagenet_429} provide the uncurated conditional generated images using the base sampler and DiffRS on the ImageNet 64$\times$64 to enable direct comparison of sample quality. For the consistency distillation model, \cref{fig:images_imagenet_cd_compare} compares the generated images of the base samplers and our method. \cref{fig:images_sd_1} provides the text-conditional generated images with a resolution of 512$\times$512 from Stable Diffusion.

\newpage
\begin{figure}[t]
    \centering
    \includegraphics[width=0.95\linewidth]{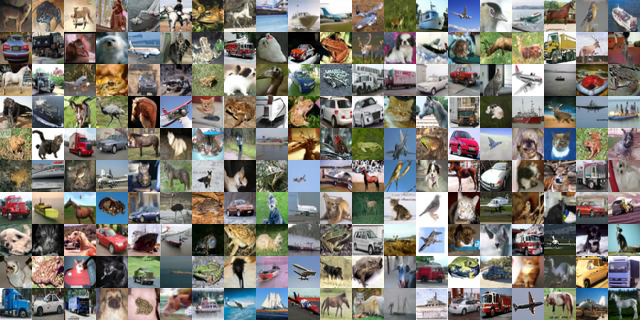}
    \caption{The uncurated generated images of DiffRS on unconditional CIFAR-10 with EDM (NFE=64.06, FID=1.59).}
    \label{fig:images_cifar_uncond}
\end{figure}

\begin{figure}[t]
    \centering
    \includegraphics[width=0.95\linewidth]{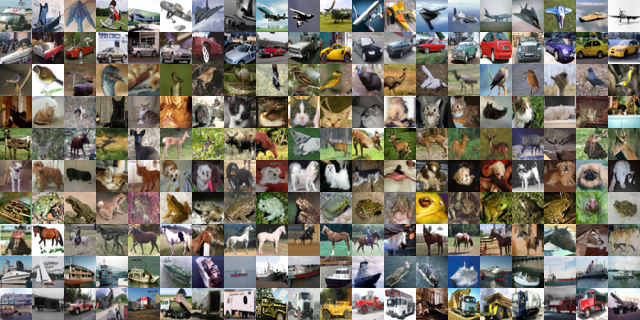}
    \caption{The uncurated generated images of DiffRS on conditional CIFAR-10 with EDM (NFE=88.22, FID=1.52).}
    \label{fig:images_cifar_cond}
\end{figure}

\begin{figure}[t]
    \centering
    \includegraphics[width=0.95\linewidth]{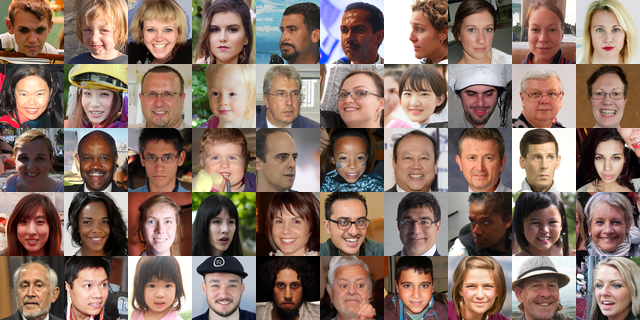}
    \caption{The uncurated generated images of DiffRS on unconditional FFHQ with EDM (NFE=198.65, FID=1.60).}
    \label{fig:images_ffhq}
\end{figure}

\begin{figure}[t]
    \centering
    \includegraphics[width=0.95\linewidth]{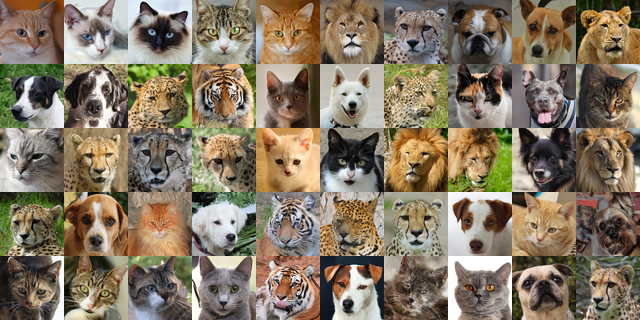}
    \caption{The uncurated generated images of DiffRS on unconditional AFHQv2 with EDM (NFE=144.92, FID=1.80).}
    \label{fig:images_afhqv2}
\end{figure}

\begin{figure}[t]
    \centering
    \begin{subfigure}[b]{0.45\linewidth}
        \centering
        \includegraphics[width=\linewidth]{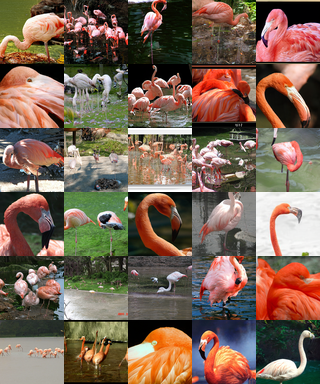}
        \caption{EDM (SDE) (NFE=127, FID=1.79)}
    \end{subfigure}
    \hfill
    \begin{subfigure}[b]{0.45\linewidth}
        \centering
        \includegraphics[width=\linewidth]{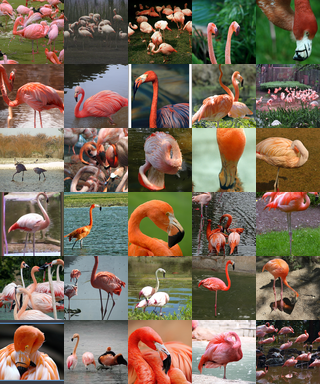}
        \caption{EDM (SDE) + DiffRS (NFE=273.93, FID=1.26)}
    \end{subfigure}
    \caption{The uncurated generated images of \texttt{flamingo} class of ImageNet 64$\times$64 with EDM.}
    \label{fig:images_imagenet_130}
\end{figure}

\begin{figure}[t]
    \centering
    \begin{subfigure}[b]{0.45\linewidth}
        \centering
        \includegraphics[width=\linewidth]{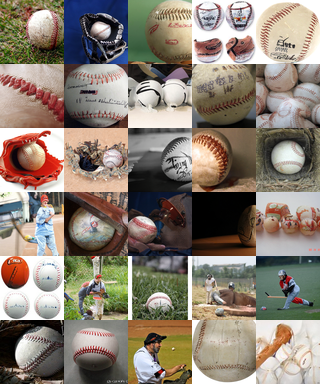}
        \caption{EDM (SDE) (NFE=127, FID=1.79)}
    \end{subfigure}
    \hfill
    \begin{subfigure}[b]{0.45\linewidth}
        \centering
        \includegraphics[width=\linewidth]{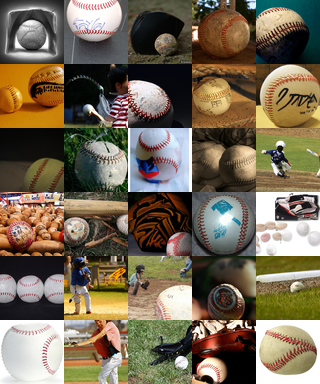}
        \caption{EDM (SDE) + DiffRS (NFE=273.93, FID=1.26)}
    \end{subfigure}
    \caption{The uncurated generated images of \texttt{baseball} class of ImageNet 64$\times$64 with EDM.}
    \label{fig:images_imagenet_429}
\end{figure}

\begin{figure}[t]
    \centering
    \begin{subfigure}[b]{0.95\linewidth}
        \centering
        \includegraphics[width=\linewidth]{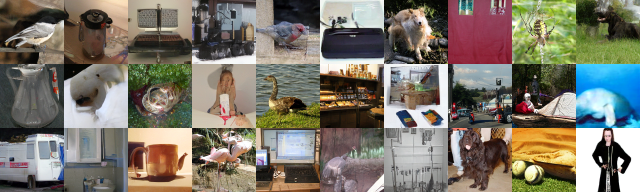}
        \caption{CD (NFE=2, FID=4.68)}
    \end{subfigure}
    \hfill
    \begin{subfigure}[b]{0.95\linewidth}
        \centering
        \includegraphics[width=\linewidth]{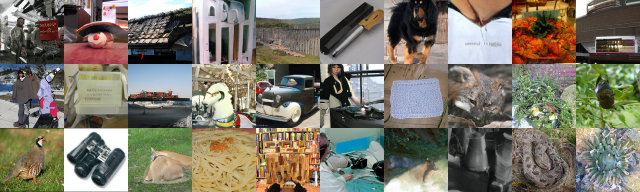}
        \caption{CD (NFE=8, FID=3.97)}
    \end{subfigure}
    \hfill
    \begin{subfigure}[b]{0.95\linewidth}
        \centering
        \includegraphics[width=\linewidth]{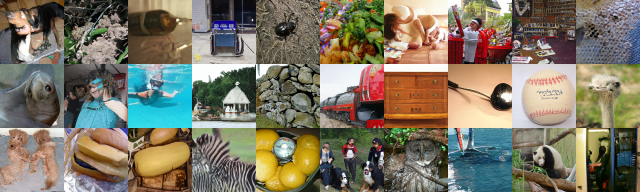}
        \caption{CD (NFE=2) + DiffRS (NFE=8.01, FID=3.07)}
    \end{subfigure}
    \caption{The uncurated generated images of (a-b) CD-based sampler and (c) DiffRS on conditional ImageNet 64$\times$64 dataset with CD.}
    \label{fig:images_imagenet_cd_compare}
\end{figure}

\begin{figure}[t]
    \centering
    \begin{subfigure}[t]{0.32\linewidth}
        \centering
        \includegraphics[width=\linewidth]{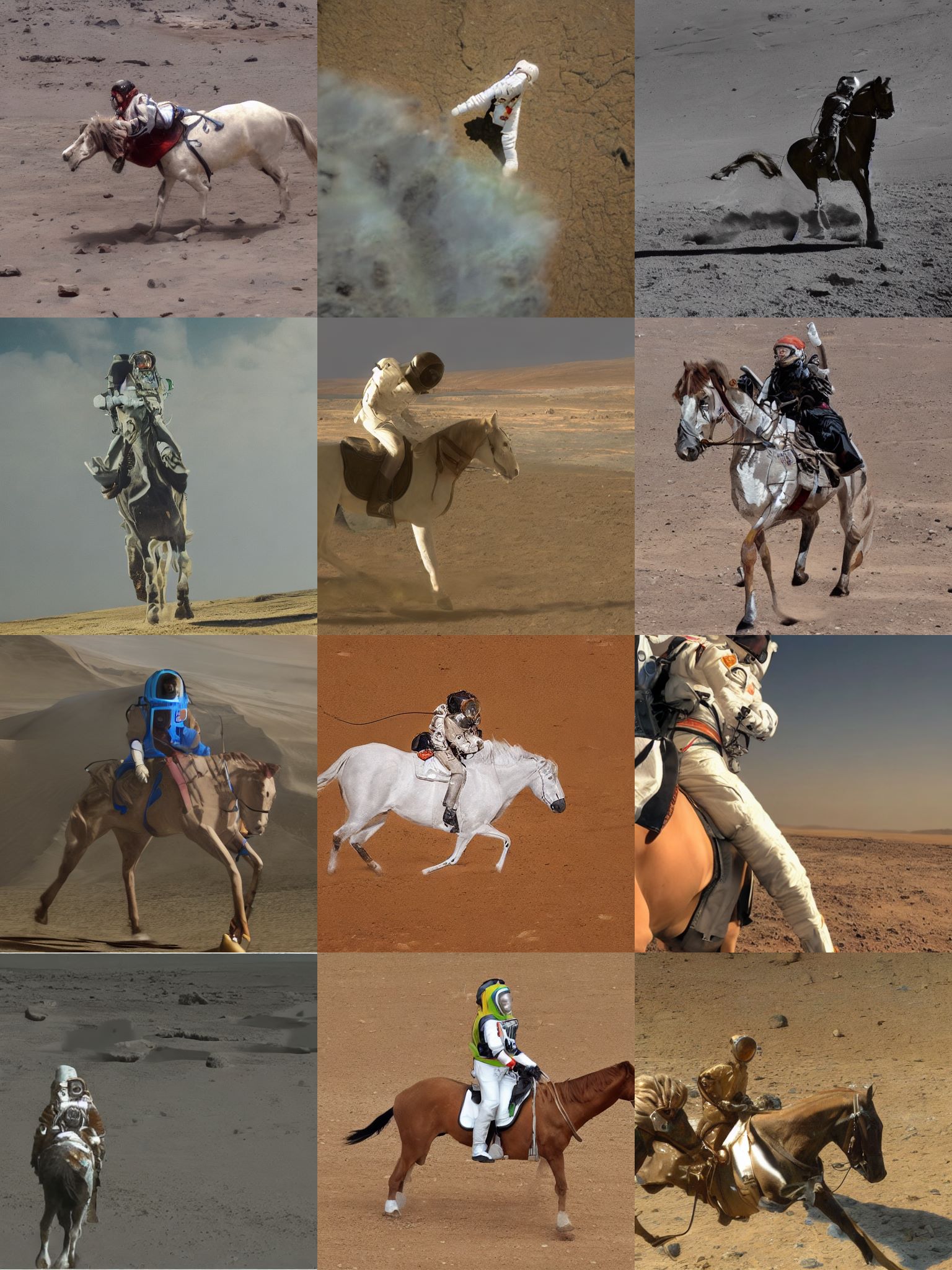}
        \caption{DDIM (NFE=100, FID=15.90)}
    \end{subfigure}
    \hfill
    \begin{subfigure}[t]{0.32\linewidth}
        \centering
        \includegraphics[width=\linewidth]{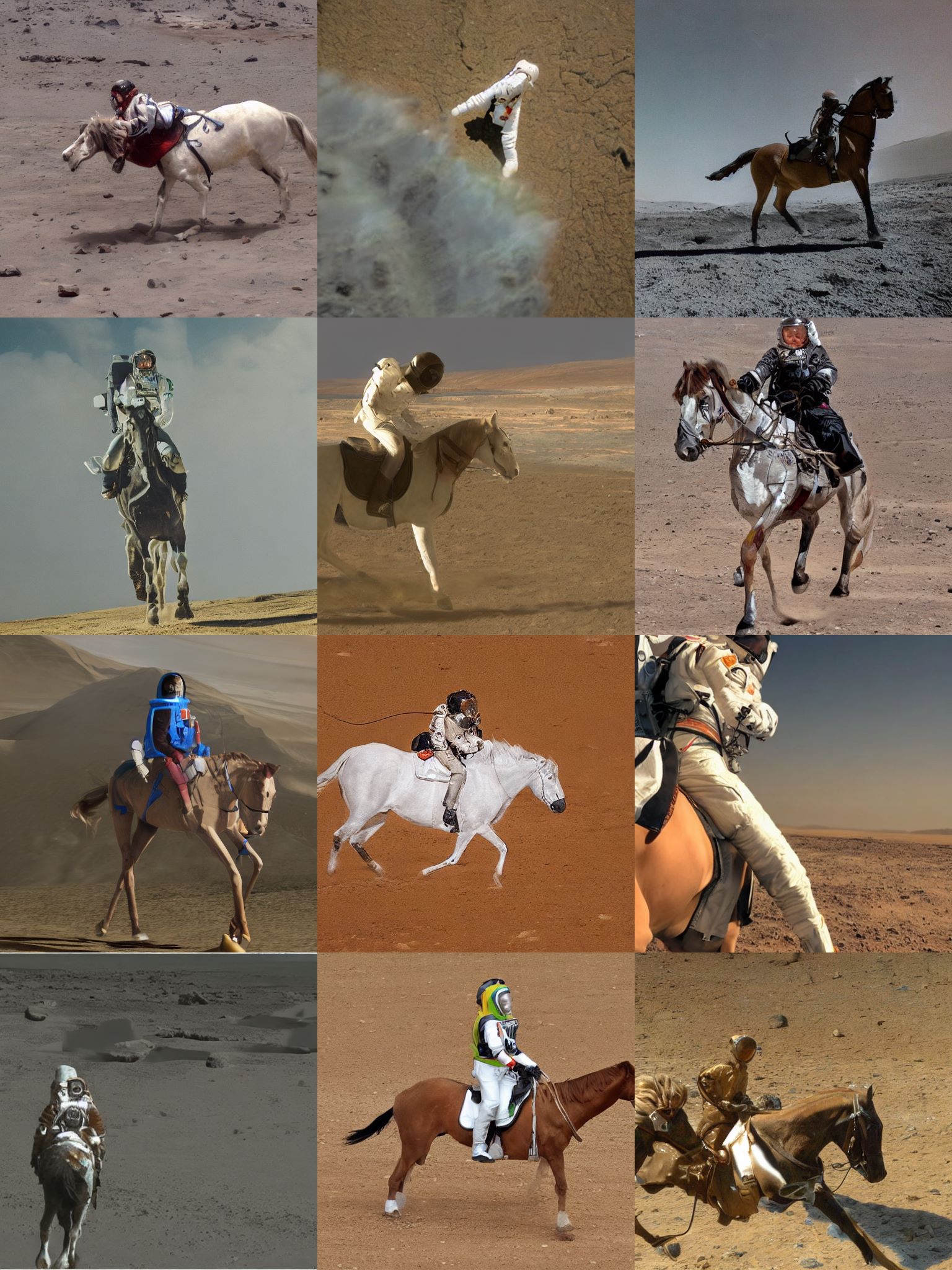}
        \caption{DDIM (NFE=200, FID=15.29)}
    \end{subfigure}
    \hfill
    \begin{subfigure}[t]{0.32\linewidth}
        \centering
        \includegraphics[width=\linewidth]{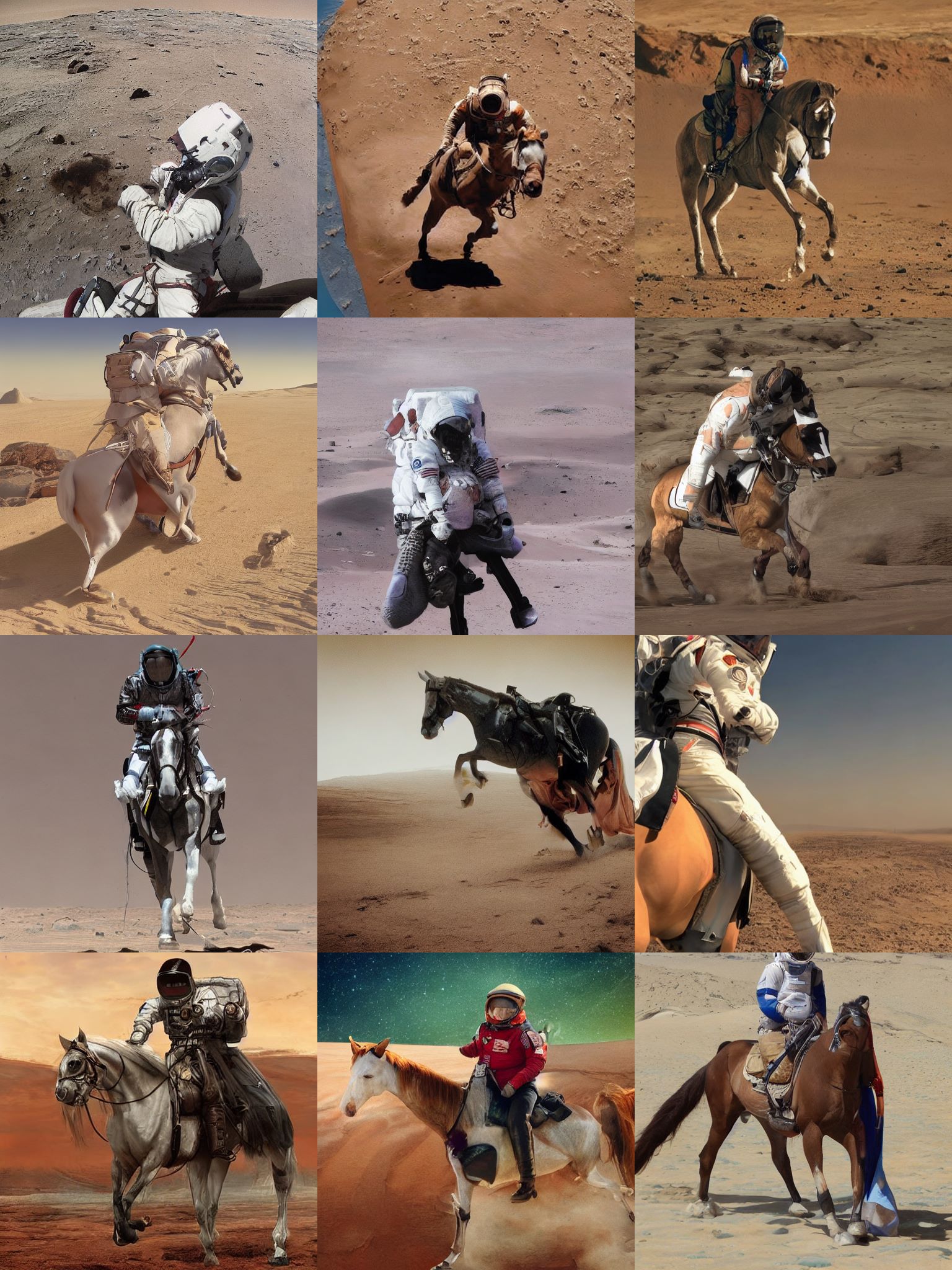}
        \caption{DDIM (NFE=100) + DiffRS (ours) \\ (NFE=166.95, FID=13.21)}
    \end{subfigure}
    \caption{The uncurated generated images, with a resolution of 512$\times$512, corresponding to the text prompt \texttt{A photo of an astronaut riding a horse on mars}, using Stable Diffusion v1.5 with a classifier-free guidance weight of 2.}
    \label{fig:images_sd_1}
\end{figure}


\end{document}

%% file: math_commands.tex
\usepackage{amsmath,amsfonts,bm}

\def\eqref#1{equation~\ref{#1}}

\def\1{\bm{1}}

\def\rvx{{\mathbf{x}}}

\def\rvz{{\mathbf{z}}}

\DeclareMathAlphabet{\mathsfit}{\encodingdefault}{\sfdefault}{m}{sl}
\SetMathAlphabet{\mathsfit}{bold}{\encodingdefault}{\sfdefault}{bx}{n}

\newcommand{\E}{\mathbb{E}}

